\documentclass[12pt]{article}

\usepackage[top=1in, bottom=1in, left=1.1in, right=1.1in]{geometry}
\usepackage{setspace}
\usepackage{amsmath,amssymb,amsfonts,amsbsy}
\usepackage{amsthm}
\usepackage[utf8]{inputenc}
\usepackage[T1]{fontenc}
\usepackage{textcomp}
\usepackage[disable]{microtype}
\usepackage{nicefrac}

\usepackage{graphicx}
\usepackage{epstopdf}
\usepackage{float}
\usepackage{afterpage}

\usepackage{booktabs}
\usepackage{dcolumn,array}
\usepackage{multirow}

\usepackage{xcolor}
\usepackage{hyperref}
\hypersetup{
  breaklinks=true,
  colorlinks=true,
  urlcolor=blue,
  linkcolor=blue,
  citecolor=blue,
}

\usepackage{authblk}

\usepackage{url}
\usepackage{algorithmic}
\usepackage{tocloft}
\usepackage{tcolorbox}
\tcbuselibrary{skins}
\usepackage{fancyhdr}

\def\figureautorefname{Fig.}
\def\sectionautorefname{Sec.}
\def\equationautorefname{Eq.}

\def\tableautorefname{Tab.}

\newenvironment{lemma}[2][Lemma]{\begin{trivlist}
\item[\hskip \labelsep {\bfseries #1}\hskip \labelsep {\bfseries #2.}]}{\end{trivlist}}

\newenvironment{proposition}[2][Proposition]{\begin{trivlist}
\item[\hskip \labelsep {\bfseries #1}\hskip \labelsep {\bfseries #2.}]}{\end{trivlist}}

\makeatletter
\renewcommand{\maketitle}{%
  \begin{center}
    {\fontsize{15}{19}\bfseries\selectfont
      AUCp: Pseudo-AUC for Inference Model Selection\\[6pt]
      with Unlabeled Validation Data in Abnormality Detection\par}
    \vspace{12pt}
    {\normalsize
      \mbox{Md~Mahfuzur~Rahman~Siddiquee}$^{1,*}$,\;
      \mbox{Fazle~Rafsani}$^{1}$,\;
      \mbox{Jay~Shah}$^{1}$,\;
      \mbox{Teresa~Wu}$^{1}$,\;
      \mbox{Catherine~D~Chong}$^{2}$,\;
      \mbox{Todd~J~Schwedt}$^{2}$,\;
      \mbox{Baoxin~Li}$^{1}$\par}
    \vspace{8pt}
    {\small\itshape
      $^{1}$School of Computing and Augmented Intelligence,
      Arizona State University, Tempe, AZ~85281\\[2pt]
      \normalfont\small\texttt{\{mrahmans, frafsani, jgshah1, teresa.wu, baoxin.li\}@asu.edu}\\[5pt]
      \itshape$^{2}$Mayo Clinic, Arizona\quad
      \normalfont\small\texttt{\{chong.catherine, schwedt.todd\}@mayo.edu}\\[7pt]
      \normalfont\footnotesize
      $^{*}$Corresponding author (\texttt{mrahmans@asu.edu}).
      Manuscript received April~3, 2025; accepted April~10, 2026.
      Supported by DoD W81XWH-15-1-0286 \& W81XWH1910534,
      NIH K23NS070891, NIH-NINDS 1R61NS113315-01, and Amgen ISS~20187183.\par}
  \end{center}
  \vspace{4pt}
}
\makeatother

\date{}

\begin{document}
\maketitle
\thispagestyle{fancy}

\begin{tcolorbox}[
  enhanced,
  colback=blue!4!white,
  colframe=blue!40!black,
  boxrule=0.6pt,
  arc=3pt,
  left=6pt, right=6pt, top=5pt, bottom=5pt,
  before skip=2pt, after skip=10pt
]
\small\textbf{Published in:}~%
Md Mahfuzur Rahman Siddiquee, F. Rafsani, J. Shah, T. Wu, C.\,D. Chong,
T.\,J. Schwedt, and B. Li,
``AUCp: Pseudo-AUC for Inference Model Selection with Unlabeled Validation
Data in Abnormality Detection,''
\textit{IEEE Transactions on Medical Imaging}, pp.~1--1, 2026.
doi:~\href{https://doi.org/10.1109/TMI.2026.3684946}{\texttt{10.1109/TMI.2026.3684946}}.
\end{tcolorbox}

\begin{abstract}
  Abnormality detection is a crucial yet challenging task in medical image analysis. Distinguishing abnormalities from normal data by learning to reconstruct normal-only data alleviates the reliance on labeled datasets. However, many studies, even if unsupervised, rely on a labeled validation set to select the best model for inference from multiple training iterations. For many diseases labeled data are unavailable and substantially time consuming to obtain. To address this,  $\boldsymbol{AUC}_p$ - a novel metric that supports abnormality detection for unsupervised and self-supervised methods is proposed. Instead of evaluating the realism of reconstructed images to select the best of model for inference, it focuses on actual detection performance and without requiring an annotated test set. Assuming the pseudo ground truth of all unannotated samples in the test set as abnormal/positive and using traditional ${AUC}$ calculation,  ${AUC}_p$ scores are derived. Given a large and representative training set of normal samples, we show mathematical and empirical evidence that model selection using ${AUC}_p$ scores improves disease detection in terms of unsupervised and self-supervised methods over conventional metrics. Using two unsupervised methods for neurologic disease detection and self-supervised methods on diverse datasets, our results demonstrate that the ${AUC}_p$ score effectively identifies the optimal model for inference, significantly enhancing abnormality and disease detection. The corresponding implementations are available in \url{https://github.com/mahfuzmohammad/AUCp}.
\end{abstract}


\bigskip
\hrule
\bigskip

\section{Introduction}
Detecting anomalies is challenging, as the lack of labeled data makes training models difficult. Selecting a well-trained and generalized model for inference is also difficult because of the absence of labeled normal and abnormal samples in the validation dataset. As a result, many studies assume that an annotated validation dataset is available ~\cite{xiang2024exploiting,zingman2024learning,frotscher2023unsupervised,georgescu2023masked,sato2023anatomy,xiang2023squid,bengs2022unsupervised,kim2021unsupervised,shvetsova2021anomaly,che2025anofpdm}, while some select models using synthetic abnormalities ~\cite{fung2023model,marimont2023harder}. Although these evaluation methods work only to compare the methods for research purposes, they keep the abnormality detection methods proposed unusable for real-life deployment. Many studies also select the inference model based on the best training loss. In ~\cite{cai2024medianomaly} the authors studied such abnormality detection methods and compared the performance of autoencoder models that trained and selected inference models using $l_1$, $l_2$, $SSIM$ (structure similarity index) and perceptual loss. They found that no metric consistently outperforms the others in multiple datasets. This is expected since these metrics aim to achieve the best reconstruction of the input image, not the best abnormality detection. 

In \cite{siddiquee2023brainomaly} and \cite{rahman2022healthygan}, the authors introduced methods to use unlabeled datasets that included normal images and abnormal images during training to improve the detection of unsupervised disease/abnormality at the patient level. In the context of abnormality detection, self-supervised learning methodologies generally involve defining a pretext task that facilitates the model's learning of significant features from the data, independently of any annotated instances similar to the unsupervised setting. In this methodology, the model acquires the ability to produce pseudo-labels by utilizing the intrinsic structures present within the data. Though these methods produce promising results, one of the major issues was unaddressed in these methods: selecting a well-trained generalized model for inference. Since the training objective in both of these methods is to generate the most realistic images possible of healthy anatomy, the model was selected based on the realism of the images generated using the Fréchet inception distance (FID)~\cite{heusel2017gans}. It is one of the most popular metrics in the generative adversarial network (GAN)  literature to measure the realism of the generated image. However, FID does not select the inference model based on its ability to detect abnormality; it is observed that the FID metric does not have a strong correlation with the underlying detection performance of the model based on its capability of abnormality detection (see~\sectionautorefname~\ref{subsec_aucp:results_model_selection}). 

Since labeled abnormal instances are not available, one cannot simply follow the validation steps of traditional supervised learning to select parameters for the model that maximize the objective function of abnormality detection. As a solution, a new metric, pseudo-AUC ($\boldsymbol{AUC}_p$) is proposed to select a suitable model for inference when the labeled validation dataset is unavailable. Given a training dataset $\boldsymbol{D}_{train}$ and a testing/evaluation dataset $\boldsymbol{D}_{test}$ (\figureautorefname~\ref{fig_aucp:dataset}), $\boldsymbol{AUC}_p$ \textit{assumes} the label for each unlabeled instance in $\boldsymbol{D}_{test}$ is ``abnormal'' while the labels for normal instances in $\boldsymbol{D}_{train}$ are known. Using these {\em imperfect} labels as pseudo ground truths in the traditional $\boldsymbol{AUC}$ calculation results in $\boldsymbol{AUC}_p$ scores. ~\sectionautorefname~\ref{sec_aucp:method} provides the details of $\boldsymbol{AUC}_p$ calculation and mathematically show that with a large amount of representative normal instances in $\boldsymbol{D}_{train}$, the $\boldsymbol{AUC}_p$ scores are very close to the actual $\boldsymbol{AUC}$ scores.

The idea of using a pseudo-AUC ($\boldsymbol{AUC}_p$) for inference-time model selection was first introduced in Brainomaly~\cite{siddiquee2023brainomaly}, where it was employed as a practical heuristic to select a checkpoint within a specific GAN-based unsupervised disease-detection pipeline. However, that work neither provided a formal definition or theoretical analysis of $\boldsymbol{AUC}_p$, nor examined its behavior beyond the Brainomaly setting or in comparison to alternative selection criteria. In contrast, the present study formalizes the $\boldsymbol{AUC}_p$ metric and establishes its theoretical relationship to the true AUC under realistic assumptions on the training data. We further provide a comprehensive empirical evaluation of $\boldsymbol{AUC}_p$ across multiple unsupervised and self-supervised abnormality-detection paradigms and diverse medical imaging datasets. In addition, we explicitly compare $\boldsymbol{AUC}_p$-based inference-time model selection with commonly used heuristics such as reconstruction-error-based selection, demonstrating that $\boldsymbol{AUC}_p$ yields more stable and reliable performance across methods and datasets.

In this study,  HealthyGAN\cite{rahman2022healthygan} and Brainomaly\cite{siddiquee2023brainomaly} are chosen as baseline methods to empirically show improvement in baseline disease detection performance using the proposed metric. Since Brainomaly, which is an extension of HealthyGAN, outperforms existing state-of-the-art unsupervised disease detection methods on a public Alzheimer's disease (AD) dataset and an institutional dataset for headache detection,  both of them are used as reference methods to compare performance with the proposed metric. Moreover, experiments on seven distinct publicly accessible datasets are conducted using four self-supervised abnormality detection approaches to evaluate the effectiveness of the proposed metric for inference model selection. In summary, this research makes five contributions: .
\begin{itemize}
    \item A new metric, $\boldsymbol{AUC}_p$ is proposed for selecting a suitable model for inference when a labeled validation dataset is unavailable.
    \item Mathematical proof is provided to show the proposed $\boldsymbol{AUC}_p$ metric provides scores close to the actual $\boldsymbol{AUC}$ scores when there is a large number of representative normal instances available in the training set by proving the pseudo-ground truths used in the $\boldsymbol{AUC}_p$ calculation are equivalent to the actual ground truth used in the calculation $\boldsymbol{AUC}$.
    \item Empirical experiments are conducted to demonstrate that $\boldsymbol{AUC}_p$ better correlates with the models' underlying disease detection performances and selects a higher-performing model than a model chosen by FID, commonly used in GAN model development.
    \item Empirical experiments are conducted to demonstrate that  $\boldsymbol{AUC}_p$ further improves baseline unsupervised (HealthyGAN \cite{rahman2022healthygan} and Brainomaly \cite{siddiquee2023brainomaly}) and self-supervised (CutPaste \cite{li2021cutpaste}, Foreign Patch Interpolation (FPI) \cite{tan2020detecting},  Poisson Image Interpolation (PII) \cite{tan2021detecting}, and Natural Synthetic Anomalies (NSA) \cite{schluter2022natural}) abnormality detection.
    \item The Brainomaly inference model selected by $\boldsymbol{AUC}_p$ is evaluated in both transductive and inductive settings to match real-world scenarios and found improved disease detection performance compared to the model previously selected by FID.
\end{itemize}

In the following sections, the mathematical formulation of the $\boldsymbol{AUC}_p$ metric is discussed in Sec. \ref{sec_aucp:method} followed by the experimental setup for both unsupervised and supervised abnormality detection methods, along with the datasets used in the experiments in Sec. \ref{sec_aucp:experiments}. Finally, the results obtained in different settings are presented and analyzed.
\section{The Proposed Metric: Pseudo-AUC ($\boldsymbol{AUC}_p$)}
\label{sec_aucp:method}

\begin{figure}[!htp]
    \centering
    \includegraphics[width=0.8\linewidth]{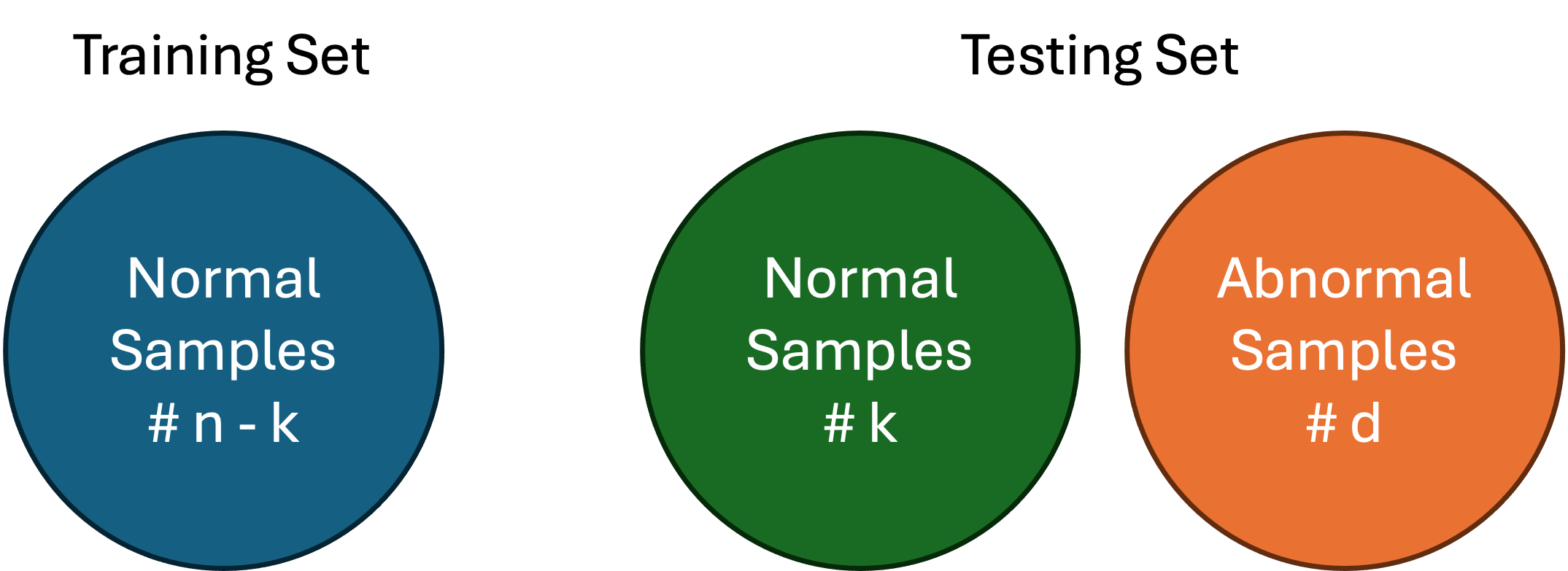}
    \caption{A visualization of the general dataset settings in abnormality detection problems. Usually, a set $\boldsymbol{D}_{train}$ of $n-k$ instances of normal anatomies is provided for training. The trained model is then applied to a testing $\boldsymbol{D}_{test}$ set of $k+d$ instances to detect abnormalities where $k$ is the number of instances from normal anatomies, and $d$ is the number of instances from abnormal anatomies. Please note that the labels of the instances in the test set are unknown.}
    \label{fig_aucp:dataset}
\end{figure}

\begin{table}[h]
\centering
\caption{Notation Summary}
\begin{tabular}{ll}
\toprule
$D_{\text{train}}$ & Training dataset containing $n - k$ normal instances \\
$D_{\text{test}}$  & Unlabeled test dataset with $k$ normal and $d$ abnormal instances \\
$n$                & Total number of normal instances across both datasets \\
$d$                & Number of abnormal instances in $D_{\text{test}}$ \\
GT                 & True labels (unknown in practice) \\
$\text{GT}_p$      & Pseudo-labels used for AUCp \\
$P$                & Model-generated anomaly scores \\
AUC                & Conventional AUC with true labels \\
$\text{AUC}_p$     & Pseudo-AUC computed with $\text{GT}_p$ \\
$M(\cdot;\theta)$  & Abnormality detection method with parameters $\theta$ \\
$f(\cdot)$         & Feature extractor \\
\bottomrule
\end{tabular}
\end{table}

Let's assume a general medical abnormality detection scenario where the training set $\boldsymbol{D}_{train}$ contains $n - k$ instances of normal anatomies for training an abnormality detection method. The set $\boldsymbol{D}_{train}$ should be comprehensive to learn the variation in the normal anatomies well. The test set $\boldsymbol{D}_{test}$ contains $k+d$ samples on which one will apply the abnormality detection method once trained. Here, $k$ is the number of instances of normal anatomies, and $d$ is the number of instances of abnormal anatomies. The total number of normal instances from both the training and testing datasets is $n$. Let assume the test set contains at least $1$ such instance of abnormal anatomies ($d >= 1$). Please note that the actual label and the distribution of the instances in the test set, being from normal and abnormal anatomies, are unknown. The objective of the abnormality detection method, such as HealthyGAN \cite{rahman2022healthygan} and Brainomaly \cite{siddiquee2023brainomaly}, is to train a model that can separate the instances of abnormal anatomies from the normal ones in the test set $\boldsymbol{D}_{test}$. A schema of the dataset is provided in~\figureautorefname~\ref{fig_aucp:dataset}. 
To calculate the $\boldsymbol{AUC}_p$ metric, pseudo-labels ($\boldsymbol{GT}_p$) are first introduced. The negative instances in the $\boldsymbol{GT}_p$ come from the given set of instances from normal/healthy anatomies, $\boldsymbol{D}_{train}$. Since no annotated positive instance (instances of diseased/abnormal anatomies) is available, we utilize the assumption in the \textit{problem settings} that the test set $\boldsymbol{D}_{test}$ contains at least 1 abnormal instances, and we assume all the instances in the $\boldsymbol{D}_{test}$ are positive/abnormal. With this, the test instances with a partially correct positive label assignment are considered as the positive instances in the $\boldsymbol{GT}_p$. Given $\boldsymbol{GT}_p$, the $\boldsymbol{AUC}_p$ for a model's prediction $\boldsymbol{P}$ is defined in~\equationautorefname~\ref{eq:aucp_def}.

\begin{equation}
\label{eq:aucp_def}
    \boldsymbol{AUC}_p(\boldsymbol{GT}_p, \boldsymbol{P}) = \boldsymbol{AUC}(\boldsymbol{GT}_p, \boldsymbol{P})
\end{equation}

\noindent Here, the $\boldsymbol{AUC}$ means the area under the receiver operating characteristics curve when the actual label ($\boldsymbol{GT}$) of the instances is known.

\begin{proposition}{1}
\label{prop_aucp:gt_vs_gtp}
    For a large set of instances from healthy/normal anatomies, $\boldsymbol{D}_{train}$, $\boldsymbol{GT}_p$ is almost equivalent to $\boldsymbol{GT}$. Here, $\boldsymbol{GT}$ would be the actual labels with $n$ negative/normal and $d$ positive/abnormal samples if labels of the test set $\boldsymbol{D}_{test}$ were available during inference model selection. In other word, when $n \to \infty$, $\boldsymbol{AUC}(\boldsymbol{GT}, \boldsymbol{GT}_p) \to 1$.
\end{proposition}
\begin{figure}[!h]
    \centering
    
    \includegraphics[scale=0.5]{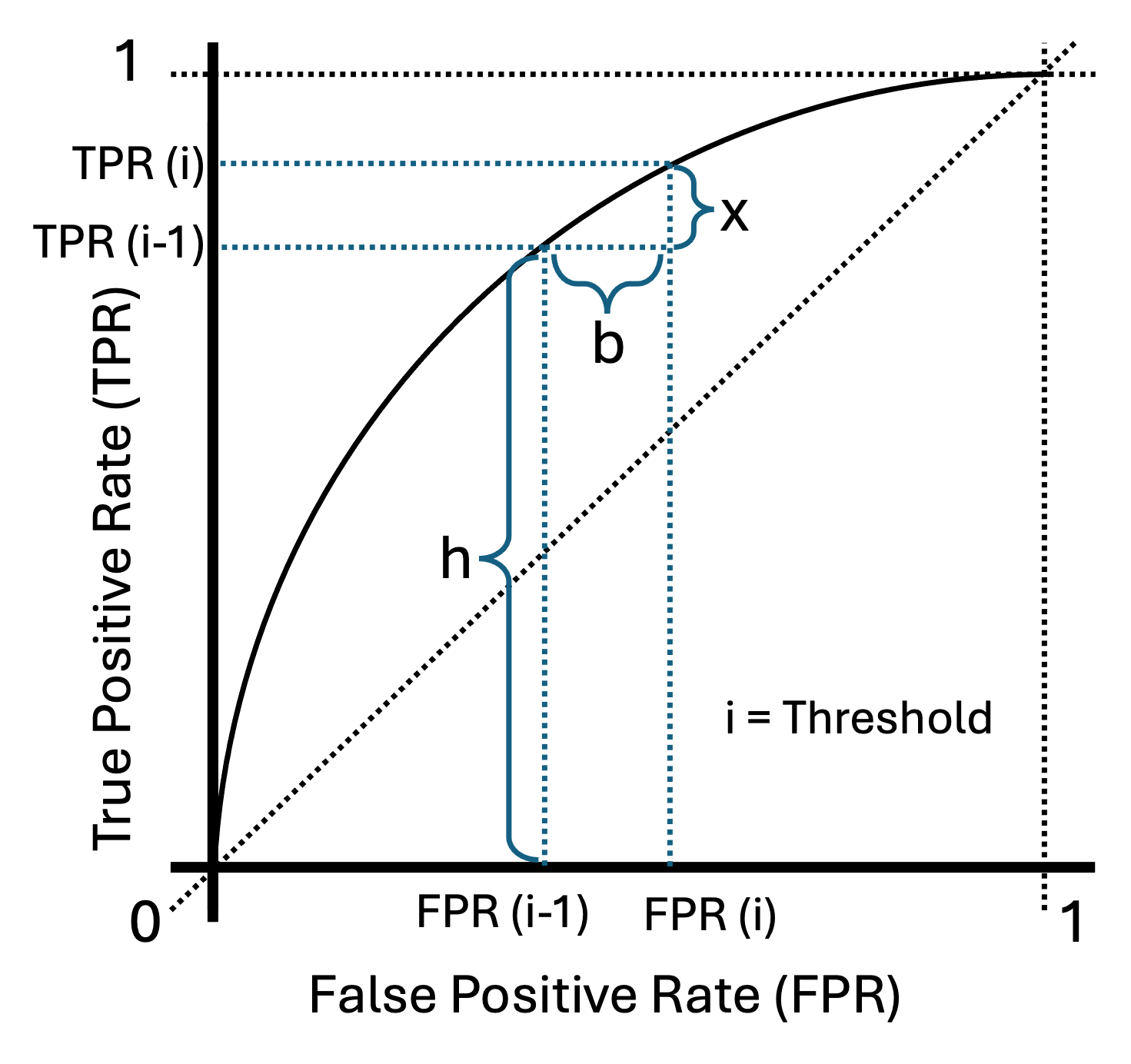}
    \caption{An illustration of ROC curve.}
    \label{fig:roc}
\end{figure}

    







\begin{proof}
Using the trapezoidal rule on the ROC curve in Fig.~\ref{fig:roc}, the area under the curve (AUC) can be written as
\begin{equation}
\operatorname{AUC}
= \frac{1}{2}\sum_{i}\bigl(\mathrm{FPR}_{i}-\mathrm{FPR}_{i-1}\bigr)\bigl(\mathrm{TPR}_{i}+\mathrm{TPR}_{i-1}\bigr),
\label{eq:auc_trap}
\end{equation}
where $\mathrm{FPR}_i=\mathrm{FP}_i/n$ and $\mathrm{TPR}_i=\mathrm{TP}_i/d$, with $n$ and $d$ denoting the numbers of negative and positive instances, respectively.

Computing $\operatorname{AUC}(\mathrm{GT},\mathrm{GT}_{p})$ yields two trapezoids (thresholds $0\!\to\!0.5$ and $0.5\!\to\!1$):
\begin{align}
\operatorname{AUC}(\mathrm{GT},\mathrm{GT}_{p})
&= \tfrac{1}{2}\bigl(\mathrm{FPR}_{0}-\mathrm{FPR}_{0.5}\bigr)\bigl(\mathrm{TPR}_{0}+\mathrm{TPR}_{0.5}\bigr) \notag\\
&\quad + \tfrac{1}{2}\bigl(\mathrm{FPR}_{0.5}-\mathrm{FPR}_{1}\bigr)\bigl(\mathrm{TPR}_{0.5}+\mathrm{TPR}_{1}\bigr).
\label{eq:two_traps}
\end{align}
Let $k$ be the number of negatives assigned positive in $\mathrm{GT}_{p}$. Then
\begin{equation}
\mathrm{FPR}_{0.5}-\mathrm{FPR}_{1}=\frac{k}{n}-\frac{k}{n}=0,
\end{equation}
so the second term in \eqref{eq:two_traps} vanishes and
\begin{align}
\operatorname{AUC}(\mathrm{GT},\mathrm{GT}_{p})
&= \tfrac{1}{2}\!\left(\frac{\mathrm{FP}_{0}-\mathrm{FP}_{0.5}}{n}\right)\!\left(\frac{\mathrm{TP}_{0}+\mathrm{TP}_{0.5}}{d}\right) \notag\\
&= \tfrac{1}{2}\!\left(\frac{n-k}{n}\right)\!\left(\frac{d+d}{d}\right)
= 1-\frac{k}{n}.
\end{align}
If $k=0$, then $\mathrm{GT}=\mathrm{GT}_{p}$. More generally, as $n\to\infty$ with fixed $k$, $\operatorname{AUC}(\mathrm{GT},\mathrm{GT}_{p})\to 1$.
\end{proof}

Consider a special scenario in which a model predicts the normal instances in the training set are normal, but the normal instances in the test set are abnormal. Mathematically, our $\boldsymbol{AUC}_p$ should provide a perfect score of 1. Such a model seems to get picked by our $\boldsymbol{AUC}_p$ over a model correctly identifying normal instances in both the training and the test set. However, if the training set contains large enough instances representative of features in negative instances, then Lemma 1 shows that the scenario when a model correctly identifies the negative instances in the training set but mislabels the negative instances in the test as positive cannot occur.

\begin{lemma}{1}
    If $\boldsymbol{D}_{train}$ is large and represents the universal feature set of negative instances (assuming $\boldsymbol{D}_{train}$ as universal negatives), then $M(f(\boldsymbol{N}_T) \cap f(\boldsymbol{D}_{train}); \theta) = M(f(\boldsymbol{N}_T); \theta)$ where $M$ is the medical abnormality detection method with the learned weight $\theta$, $\boldsymbol{N}_T$ is the set of normal instances ($k$ instances and $k > 0$) in the test set $\boldsymbol{D}_{test}$, and $f$ is a feature extractor from the input instances.
\end{lemma}

\begin{proof}
    Let's assume $M(f(\boldsymbol{N}_T) \cap f(\boldsymbol{D}_{train}); \theta) \ne M(f(\boldsymbol{N}_T); \theta)$, which means the abnormality detection provides different predictions for each case. It is only possible in two cases:

    \begin{description}
        \item[Case 1:] $f(\boldsymbol{N}_T) \cap f(\boldsymbol{D}_{train}) = \emptyset$. However, it is impossible as $\boldsymbol{D}_{train}$ is not empty and represents the universal feature set of negative instances. Hence, $f(\boldsymbol{N}_T) \subseteq f(\boldsymbol{D}_{train})$ and $f(\boldsymbol{N}_T) \cap f(\boldsymbol{D}_{train}) \ne \emptyset$.

        \item[Case 2:] $|f(\boldsymbol{N}_T) \cap f(\boldsymbol{D}_{train})| < |f(\boldsymbol{N}_T)|$. Similar to \textit{Case 1}, this is also impossible as $\boldsymbol{D}_{train}$ represents the universal feature set of negative instances.
    \end{description}

    \noindent Therefore, $f(\boldsymbol{N}_T) \cap f(\boldsymbol{D}_{train}) = f(\boldsymbol{N}_T)$ which implies $M(f(\boldsymbol{N}_T) \cap f(\boldsymbol{D}_{train}); \theta) = M(f(\boldsymbol{N}_T); \theta)$.
\end{proof}

\begin{figure*}[!htp]
    \centering
    \includegraphics[width=1.0\linewidth, height=0.25\textheight]{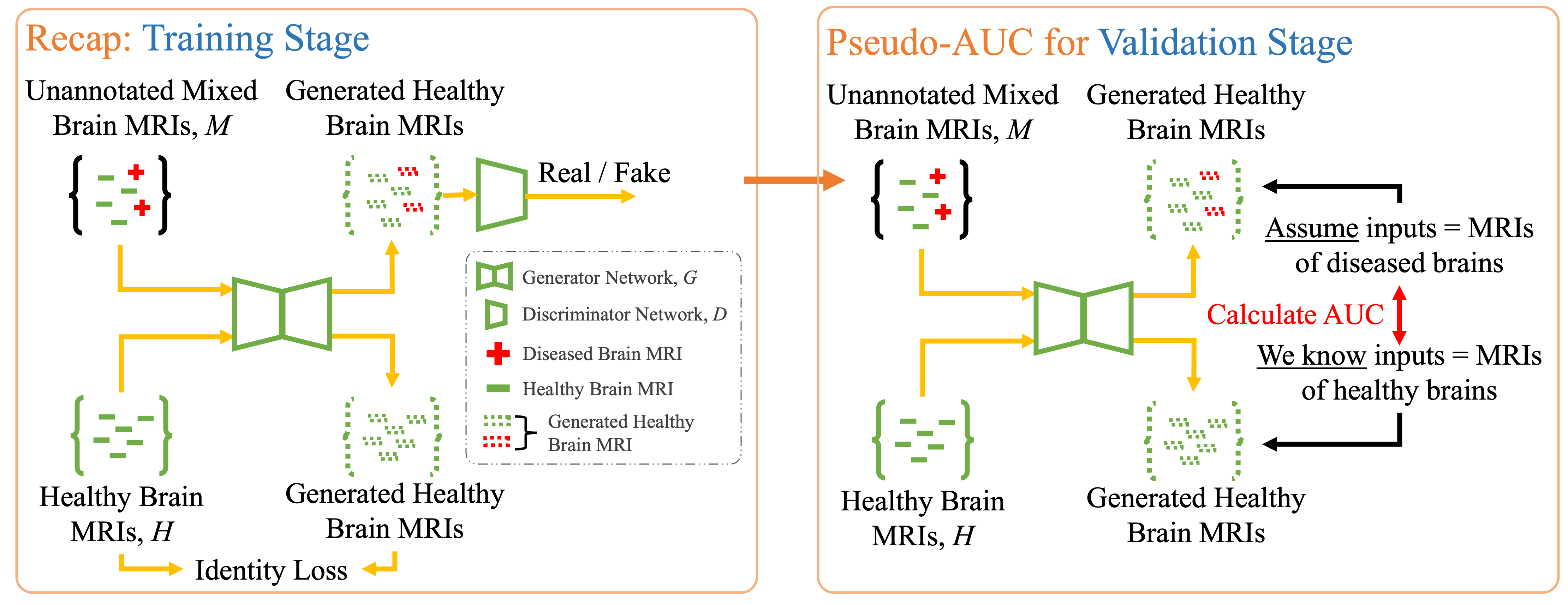}
    \caption{Overview of the proposed $\boldsymbol{AUC}_p$ metric calculation using Brainomaly.}
    \label{fig_aucp:aucp_overview}
\end{figure*}

An example of $\boldsymbol{AUC}_p$ calculation using the Brainomaly method is provided in~\figureautorefname~\ref{fig_aucp:aucp_overview}. Brainomaly learns from the given T1-weighted brain MRIs by iteratively going over the entire dataset. Here, the $\boldsymbol{D}_{train}$ contains T1-weighted MRI of healthy brains, and $\boldsymbol{D}_{test}$ contains an unlabeled mixture of T1-weighted images of both healthy and diseased brains. During training, Brainomaly produces a model after a fixed set of iterations. $\boldsymbol{AUC}_p$ scores for each model are calculated and the model producing the highest score is selected as the inference model. To calculate $\boldsymbol{AUC}_p$,  the {\em disease detection scores}, as discussed in \cite{siddiquee2023brainomaly}, for each model is generated. Since it is known that the set $\boldsymbol{D}_{train}$ contains only healthy brain MRIs, we assume the labels for MRIs in the unannotated mixed brain MRI set, $\boldsymbol{D}_{test}$. Using these {\em imperfect} annotations as ground truths $\boldsymbol{GT}_p$ along with the {\em disease detection scores} in the traditional $\boldsymbol{AUC}$ calculation,  $\boldsymbol{AUC}_p$ scores (\equationautorefname~\ref{eq:aucp_def}) are calculated. In~\sectionautorefname~\ref{subsec_aucp:results_model_selection}, it is shown that $\boldsymbol{AUC}_p$ selects a better-performing model for inference compared to FID, commonly used in existing works~\cite{rahman2022healthygan}.
\section{Experiment Design}
\label{sec_aucp:experiments}

\subsection{Datasets for Unsupervised Methods}

The same datasets as \cite{siddiquee2023brainomaly} is used for evaluation. The description of each dataset is provided below.

\subsubsection{Alzheimer's Disease Dataset}

This dataset is obtained from the ADNI database \href{http://adni.loni.usc.edu/}{(adni.loni.usc.edu)}, which is a large-scale public repository of clinical, neuropsychological, behavioral, genetic, and neuroimaging data to track the progression of Alzheimer's disease dementia. Using data from 3 studies that ADNI offers, ADNI-1, ADNI-2, and ADNI-GO,  T1 MRI scans of 536 Alzheimer's disease patients{\iffalse, 2606 mild cognitive impairment (MCI) patients,\fi} and 1271 healthy controls are collected and processed.

For $\boldsymbol{D}_{train}$,  501 MRIs from healthy controls are randomly selected. Two unlabeled mixed brain MRI sets ($\boldsymbol{D}_{test}$): {\em AD DS1} and {\em AD DS2}, meaning two instances of $\boldsymbol{D}_{test}$, are created. Each contains 268 MRIs from patients with Alzheimer's disease and 385 MRIs from healthy controls. Having two instances of $\boldsymbol{D}_{test}$ helps evaluate the performance of Brainomaly's model selected by $\boldsymbol{AUC}_p$ for Alzheimer's disease detection in both transductive and inductive settings in~\sectionautorefname~\ref{subsec_aucp:results_trans_vs_ind}.

All 3D MRIs in this dataset are registered to the {\em MNI152 1mm} template and skull stripped. The 3D MRIs are saved as 2D sagittal slices because all the competing methods perform prediction for each 2D slice. The slice-level predictions are aggregated by averaging them for patient-level predictions during evaluation.

\subsubsection{Headache Dataset}

The institutional dataset includes MRIs of 96 individuals with migraine, 48 with acute post-traumatic headache (APTH), and 49 with persistent post-traumatic headache (PPTH), diagnosed according to the International Classification of Headache Disorders (ICHD)\cite{IHS2018} diagnostic criteria. In addition MRIs of 104 healthy controls are collected. The institutional dataset is extended by including additional MRIs of 428 healthy controls from the publicly available IXI dataset~\cite{ixi} for a fair comparison with ~\cite{rahman2023headache}.

For the experiments, the model is trained by combining all headache types into one group first and then investigated each subgroup's performance separately. Total 232 MRIs of healthy controls are randomly selected as  $\boldsymbol{D}_{train}$. Similiar to the Alzheimer's disease experiment, two sets of unlabeled mixed brain MRI sets ($\boldsymbol{D}_{test}$): {\em HEAD DS1} and {\em HEAD DS2} are created. Each contains an equal number of MRIs for migraine (n = 48), APTH (n = 24), and healthy controls (n = 150). 24 out of 49 MRIs of those with PPTH were included in {\em HEAD DS1} and the rest in {\em HEAD DS2}. In addition, two sets of $\boldsymbol{D}_{test}$ are created to evaluate the performance of Brainomaly's model selected by the $\boldsymbol{AUC}_p$ for headache detection in both transductive and inductive settings in~\sectionautorefname~\ref{subsec_aucp:results_trans_vs_ind}.

All 3D MRIs in this dataset were registered to the {\em MNI152 1mm} template and skull stripped. Like the Alzheimer's disease dataset,  the 3D MRIs are converted and saved as 2D sagittal slices. And the slice-level predictions are aggregated by averaging them for patient-level predictions during evaluation.

\subsection{Datasets for Self-supervised \& Reconstruction-Based Methods}
\label{subsec_exp:self-supervosed-datasets}
\subsubsection{Pulmonary Disease Dataset: RSNA, VinDr-CXR} The RSNA dataset \cite{rsna-pneumonia-detection-challenge} used in the experiments consists of a large collection of chest X-ray images with two distinct classes: normal images and images showing lung opacities. Lung opacities are areas of increased radiographic density that may indicate pneumonia or other pulmonary pathologies, making them a clinically significant marker for abnormality detection. In this study, the dataset comprised of 8851 normal images and 6012 X-rays with lung opacities. For the purpose of abnormality detection, the model was trained exclusively on normal images. Specifically, 3851 normal X-rays are used to construct the training dataset $\boldsymbol{D}_{train}$. This training approach allows the model to learn the underlying distribution of healthy lung patterns, ensuring that any deviation from these learned features in test images can be flagged as anomalous. For evaluation, a balanced test set $\boldsymbol{D}_{test}$ consisting of 1000 normal images and 1000 lung opacity images is used.

 VinDr-CXR \cite{vinbigdata-chest-xray-abnormalities-detection} is a comprehensive chest X-ray dataset featuring richly annotated images that capture a broad spectrum of pulmonary conditions through expert evaluations of 14 distinct abnormalities comprising of 10,606 normal and 4,394 abnormal chest X-rays. This diverse imaging resource supports clinical research by providing high-quality radiological data that reflect both typical and pathological findings in the lungs. In the experiments, we implement a one-class learning approach by training solely on healthy examples using 4000 normal images as our training dataset $\boldsymbol{D}_{train}$ so the model can internalize the subtle patterns of normal lung anatomy. The model’s abnormality detection capabilities are then rigorously evaluated on a test set $\boldsymbol{D}_{test}$ consisting of 1000 normal images and 1000 abnormal images.

\subsubsection{Brain Tumor Dataset: Brain, BraTS21} The Brain dataset \cite{msoud_nickparvar_2021} comprises 2000 MRI slices devoid of tumors, 1621 exhibiting glioma, and 1645 displaying meningioma.
 Images that are devoid of tumors are sourced from Br35H5 \cite{hamada_brain_tumor_2021} and Saleh et al. \cite{9328072}, while the images depicting glioma and meningioma are derived from Saleh et al. \cite{9328072} and Cheng et al. \cite{10.1371/journal.pone.0140381}. For this study, 1000 images that do not show any cancer are used as $\boldsymbol{D}_{train}$ and 600 images that do show cancer along with 600 images that show cancer (300 with glioma and 300 with meningioma) as $\boldsymbol{D}_{test}$.

 The Brats2021 \cite{baid2021rsna} dataset offers 1251 MRI instances with a resolution of 155 × 240 × 240, along with appropriate voxel-level annotations for tumor areas. Each example comprises many modalities: T1, contrast enhanched T1, T2, and FLAIR.  FLAIR is studied in experiments due to its heightened sensitivity to tumor locations. During preprocessing, the scans are subjected to a central cropping operation to eliminate the vacant corners, yielding dimensions of 70 × 208 × 208. Total 1051 images are used for training slice extraction and the remaining 200 scans for testing slice extraction. To create $\boldsymbol{D}_{train}$, equally sample slices from the original normal axial slices devoid of malignancies in the scans are used, assuring balanced representation. Furthermore,  superfluous slices are removed to maintain a minimum interval of five slices between each sampled slice, ensuring adequate content diversity among adjacent slices. A total of 4,211 normal 2D axial slices have been taken from 1,051 FLAIR images. In accordance with the same idea, 828 normal and 1948 tumor slices from the remaining 200 FLAIR images are extracted to construct $\boldsymbol{D}_{test}$.
\subsubsection{Retinal Fundus Dataset: LAG} The LAG dataset \cite{8953932} comprises 5,824 color retinal fundus images obtained from Beijing Tongren Hospital, including 2,392 confirmed positive and 3,432 negative glaucoma samples. Each image is rigorously diagnosed by qualified glaucoma specialists who evaluate both morphological and functional aspects such as intra-ocular pressure, visual field loss, and manual optic disc assessment to establish a gold standard binary label for glaucoma. In accordance with the Helsinki Declaration and owing to the retrospective, fully anonymized nature of the dataset, patient consent was not required. Additionally, an innovative eye-tracking experiment was conducted: each initially blurred fundus image is presented to ophthalmologists, who use mouse clicks to successively clear circular regions (with a fixed radius of 40 pixels on 500 × 500 images) to reveal regions of interest. The order and coordinates of these clearings are then used to generate an attention map via a 2D Gaussian filter ($\sigma = 25$), capturing the diagnostic focus of the specialists. In the experiments, 1500 normal images are used as $\boldsymbol{D}_{train}$, and 811 normal and 811 anomalous images as $\boldsymbol{D}_{test}$.
\subsubsection{Skin Lesion Dataset: ISIC2018} The ISIC 2018 dataset \cite{DBLP:journals/corr/abs-1902-03368}, provided as part of Task 3 of the ISIC 2018 challenge, is a large-scale, publicly available collection of dermoscopic images curated for skin lesion analysis. It comprises seven lesion categories, with the nevus (NV) class designated as the normal category, while the remaining six categories represent various abnormal lesions such as melanoma, basal cell carcinoma, and others. The dataset features high-quality annotations from expert dermatologists, ensuring that each image is accurately labeled for clinical relevance. In the experiments, 6705 normal (NV) images from the official training set are used as $\boldsymbol{D}_{train}$ to model the typical features of benign lesions. For evaluation, a test set $\boldsymbol{D}_{test}$ consisting of 909 normal (NV) images along with 603 abnormal images drawn from the six non-NV categories is constructed, providing a balanced framework to assess the performance of the proposed abnormality detection approach.
\subsubsection{Lymph Node Metastasis Dataset: Camelyon16} A refined version of Camelyon16 \cite{bejnordi2017diagnostic} for Alzheimer's disease (AD) as provided by Bao et al.\cite{bao2024bmad}, merging their validation and testing sets is used as  $\boldsymbol{D}_{test}$. Whole slide images (WSIs) are segmented into 256 × 256 patches at 40x magnification, which are randomly chosen to create an image dataset for AD. In the experiment,  5088 normal images are used as  $\boldsymbol{D}_{train}$, and 1120 normal images along with 1113 abnormal images as $\boldsymbol{D}_{test}$.

\subsection{Evaluation Methods for Application of $\boldsymbol{{AUC}_p}$}

For an empirical analysis of $\boldsymbol{AUC}_p$, we have applied it to the  reference abnormality detection methods, HealthyGAN and Brainomaly. First, the abnormality detection performance of both the methods for Alzheimer's disease (\sectionautorefname~\ref{subsec_aucp:ad_results}) and headache (\sectionautorefname~\ref{subsec_aucp:headache_results}) detection are evaluated. The detection performance of their models selected by the $\boldsymbol{AUC}_p$ with their models previously selected by FID in \cite{rahman2022healthygan} and \cite{siddiquee2023brainomaly} are compared, respectively. Five state-of-the-art unsupervised abnormality detection methods are included as the benchmark. Among these, DDAD~\cite{cai2022dual}, HealthyGAN~\cite{rahman2022healthygan}, and Brainomaly~\cite{siddiquee2023brainomaly} utilize unlabeled instances during training. On the other hand, ALAD~\cite{zenati2018adversarially}, ALOOC~\cite{sabokrou2018adversarially}, f-AnoGAN~\cite{schlegl2019f}, and Ganomaly~\cite{akcay2018ganomaly} learn only from images of healthy subjects. In addition, we compared the proposed $\boldsymbol{AUC}_p$ metric with FID (\sectionautorefname~\ref{subsec_aucp:results_model_selection}) using models selected from Brainomaly and analyzed Brainomaly's performance in transductive and inductive learning settings (\sectionautorefname~\ref{subsec_aucp:results_trans_vs_ind}) using the model selected by $\boldsymbol{AUC}_p$.

 For the self-supervised methods namely CutPaste \cite{li2021cutpaste}, Foreign Patch Interpolation (FPI) \cite{tan2020detecting}, Poisson Image Interpolation (PII) \cite{tan2021detecting}, and Natural Synthetic Anomaly (NSA) \cite{schluter2022natural} model performance using the $\boldsymbol{AUC}$ metric, following the approach described in \cite{cai2024medianomaly} is evaluated. We also evaluated different image-reconstruction based methods (AE-U ~\cite{mao2020abnormality}, MemAE \cite{gong2019memorizing}, AE, AE-L1, AE-SSIM \cite{bergmann2018improving}, and AE-Perceptual \cite{shvetsova2021anomaly}) and feature-reconstruction methods (FAE-SSIM, FAE-MSE \cite{meissen2022unsupervised}) which are also described in \cite{cai2024medianomaly}. However, unlike \cite{cai2024medianomaly}, where the inference model is selected after training for a fixed number of epochs, a dynamic selection strategy is adopted. Specifically, the inference model is chosen based on the highest $\boldsymbol{AUC}_p$ score observed across all training epochs, ensuring that the selected model represents the optimal performance achieved during training. The evaluations of these self-supervised and reconstruction-based methods across seven publicly available datasets (Sec. \ref{subsec_exp:self-supervosed-datasets}) and the results are presented in Sec. \ref{subsec:self_supervised_results}

\subsection{Implementation Details}

Similar to \cite{rahman2022healthygan} and \cite{siddiquee2023brainomaly}, all the models operate on 2D MRI slices (sagittal slices) for all experiments. A central crop is performed to remove empty regions outside the brain, resulting in $192\times192$ sagittal slices for both datasets. For training HealthyGAN and Brainomaly,  $\lambda_{id}$ = 1 and a batch size of 16. For HealthyGAN,   $\lambda_{gp} = 10$, $\lambda_{rec} = 1$, $\lambda_{f} = 0.1$, $\lambda_{fz} = 1$, and $\lambda_{fs} = 1$. Both HealthyGAN and Brainomaly are trained for 400,000 iterations and saved a model for $\boldsymbol{AUC}_p$ calculation after every 10,000 iterations. Adam optimizer is used for both methods with a learning rate of $1e^{-4}$. The learning rate has been decayed for the last 100,000 iterations.

Following \cite{cai2024medianomaly}, we implemented seven one-stage self-supervised approaches that learn from synthetically generated anomalies and are then applied to real anomalies: CutPaste-Normal, CutPaste-3way, and CutPaste (variants implemented as in \cite{li2021cutpaste}); Foreign Patch Interpolation (FPI) \cite{tan2020detecting}, which replaces a local patch with one from another normal image; FPI-Poisson, which augments FPI by blending pasted patches via Poisson editing as in \cite{tan2021detecting}; Poisson Image Interpolation (PII) \cite{tan2021detecting}, which directly uses Poisson blending to create seamless synthetic anomalies; and Natural Synthetic Anomalies (NSA) \cite{schluter2022natural}, which injects structured perturbations to produce diverse, realistic pseudo-anomalies. We include two feature-space autoencoder baselines, FAE-SSIM and FAE-MSE, which reconstruct deep features of normal images and score anomalies by the residual between original and reconstructed features (measured with SSIM or MSE, respectively). We assess six image-reconstruction based autoencoders: a vanilla AE, its loss-specific variants AE-L1 and AE-SSIM, a perceptual variant AE-Perceptual, the MemAE model with a memory module that constrains reconstruction to normal patterns, and AEU, which incorporates uncertainty to modulate reconstruction error. All image-reconstruction models are trained on normal data; anomaly scores are derived from pixel/feature residuals consistent with prior work. For every method, performance is reported with the $\boldsymbol{AUC}$ metric.

\section{Results and Analyses}
\label{sec_aucp:results}

\begin{table*}[!htp]
    \caption{Comparing Alzheimer's disease and headache detection performance of HealthyGAN and Brainomaly's models selected by $\boldsymbol{AUC}_p$ and FID with five state-of-the-art abnormality detection methods. The models used to report the number in {\em DS1} columns used {\em DS2} as the unlabeled mixed data for training and $\boldsymbol{AUC}_p$ calculation, but the detection performance was calculated on {\em DS1}. Therefore, the detection performances reported in this table are from an inductive setting, which means the models were applied to test data that was not part of the training process. We also report the detection performance of the model when trained and tested on {\em DS2}, known as transductive setting, in~\sectionautorefname~\ref{subsec_aucp:results_trans_vs_ind}. Similarly, the columns {\em DS2} report detection performance on {\em DS2} split in an inductive setting. Numbers in \textbf{boldface} indicate the best results, and \underline{underlined} numbers indicate the second-best results. As seen, our $\boldsymbol{AUC}_p$ metric helps both HealthyGAN and Brainomaly improve their detection performances.}
    \centering
    \begin{tabular}{l | c | c | c | c | c | c}
        \hline
        \multirow{2}{*}{Methods} & \multicolumn{3}{c |}{\textbf{Alzheimer's Dataset}} & \multicolumn{3}{c}{\textbf{Headache Dataset}} \\
        \cline{2-7}
         & {\em AD DS1} & {\em AD DS2} & \textbf{{\em Avg}} & {\em HEAD DS1} & {\em HEAD DS2} & \textbf{{\em Avg}}\\
        \hline
        ALAD \cite{zenati2018adversarially} & 0.5127 & 0.5222 & 0.5175 & 0.6920 & 0.6990 & 0.6955 \\
        ALOOC \cite{sabokrou2018adversarially} & 0.4746 & 0.4670 & 0.4708 & 0.6566 & 0.3044 & 0.4805 \\
        f-AnoGAN \cite{schlegl2019f} & 0.6093 & 0.5946 & 0.6020 & 0.3925 & 0.4354 & 0.4071 \\
        Ganomaly \cite{akcay2018ganomaly} & 0.6048 & 0.5864 & 0.5956 & 0.6514 & 0.7313 & 0.6913 \\
        DDAD \cite{cai2022dual} & 0.5955 & 0.5897 & 0.5926 & 0.6431 & 0.6128 & 0.6280 \\
        \hline
        HealthyGAN & 0.4598 & 0.5222 & 0.4910 & 0.7410 & 0.7979 & 0.7695 \\
        + $\boldsymbol{AUC}_p$ & 0.5905 & 0.6034 & 0.5970 & 0.8276 & 0.7899 & 0.8088 \\
        \hline
        Brainomaly & \underline{0.6389} & \underline{0.6453} & \underline{0.6421} & \underline{0.9002} & \underline{0.8589} & \underline{0.8796} \\
        \textbf{+ $\boldsymbol{AUC}_p$} & \textbf{0.6452} & \textbf{0.6648} & \textbf{0.6550} & \textbf{0.9041} & \textbf{0.8878} & \textbf{0.8960} \\
        \hline
    \end{tabular}
    \label{tab_aucp:detection_results}
\end{table*}
\subsection {Unsupervised Methods}
\subsubsection{Alzheimer's Disease Detection}
\label{subsec_aucp:ad_results}

\tableautorefname~\ref{tab_aucp:detection_results} compares Alzheimer's disease (AD) detection performance of models from both HealthyGAN and Brainomaly selected using FID and $\boldsymbol{AUC}_p$ with five state-of-the-art methods: ALAD \cite{zenati2018adversarially}, ALOOC \cite{sabokrou2018adversarially}, f-AnoGAN \cite{schlegl2019f}, GANomaly \cite{akcay2018ganomaly}, DDAD \cite{cai2022dual}. The models used to report the number in {\em AD DS1} column used {\em AD DS2} as the unlabeled mixed data for training and $\boldsymbol{AUC}_p$ calculation, but the detection performance was calculated on {\em AD DS1}. Therefore, the detection performances reported in~\tableautorefname~\ref{tab_aucp:detection_results} are from an inductive setting, which means the models were applied to test data that were not part of the training process.  The detection performance of the model when trained and tested on {\em AD DS2}, known as transductive setting, in~\sectionautorefname~\ref{subsec_aucp:results_trans_vs_ind} are also reported. Similarly, the column {\em AD DS2} reports detection performance on {\em AD DS2} split in an inductive setting. As seen, the $\boldsymbol{AUC}_p$ improved Alzheimer's disease detection performance for both HealthyGAN and Brainomaly. The margin of improvement is much higher for HealthyGAN, improving by +21.59\% on average for AD detection. Models selected by $\boldsymbol{AUC}_p$ for Brainomaly improved its AD detection performance by +2.01\% on average, keeping it the best-performing method in this benchmark. Among the competing five state-of-the-art abnormality detection methods, f-AnoGAN performs the best AD detection with an AUC score of 0.6020. Brainomaly's model selected by $\boldsymbol{AUC}_p$ outperforms f-AnoGAN for AD detection by a large +8.8\% margin.

\subsubsection{Headache Detection}
\label{subsec_aucp:headache_results}

\tableautorefname~\ref{tab_aucp:detection_results} also provides a comprehensive analysis of the efficacy of HealthyGAN and Brainomaly's model selected using FID and $\boldsymbol{AUC}_p$ in detecting headaches, juxtaposed against five state-of-the-art methods. The headache detection performances showcased in the {\em HEAD DS1} column were trained using {\em HEAD DS2} as the source for unlabeled mixed data, facilitating training and $\boldsymbol{AUC}_p$ calculation. However, the ultimate assessment of detection performance was conducted on the {\em HEAD DS1} dataset, thereby reflecting an inductive setting where models were confronted with unseen test data.

The detailed evaluation extends to the transductive setting as well, elaborated upon in~\sectionautorefname~\ref{subsec_aucp:results_trans_vs_ind}, where models were trained and tested on the {\em HEAD DS2} dataset. Correspondingly, the column {\em HEAD DS2} reports detection performance on {\em HEAD DS2} split in an inductive setting.

Notably, the adoption of $\boldsymbol{AUC}_p$ significantly enhances the detection prowess of both HealthyGAN and Brainomaly, albeit with varying degrees of improvement. While HealthyGAN exhibits a remarkable advancement, with an average enhancement of +5.11\% in headache detection, Brainomaly experiences a more modest increase of +1.86\%, yet it continues to assert its dominance as the top-performing method in this benchmark.

Among the competing state-of-the-art methods, ALAD emerges as a standout performer in headache detection with an AUC score of 0.6955. However, the model selected by $\boldsymbol{AUC}_p$ from Brainomaly eclipses ALAD's performance by a substantial margin of +28.83\%.

\subsubsection{Transductive vs. Inductive Learning}
\label{subsec_aucp:results_trans_vs_ind}

Using an unlabeled set of mixed brain MRIs (\figureautorefname~\ref{fig_aucp:aucp_overview}) allows Brainomaly to operate in transductive and inductive learning modes. In \cite{siddiquee2023brainomaly}, it is shown that Brainomaly performs statistically similar in both operating settings. Here, the model's performance selected by $\boldsymbol{AUC}_p$ in both transductive and inductive settings are evaluated and compared with the model previously selected by FID. Please note that for the transductive learning setting,  both AD and headache detection on the unlabeled mixed brain MRI set used during training are evaluated. In contrast, for the inductive learning setting, an additional unseen test sets for AD detection and headache detection evaluation are utilized. Please also note that the performance reported in~\tableautorefname~\ref{tab_aucp:detection_results} and analyzed in the previous two subsections were in inductive settings.

\begin{table*}[!htp]
    \caption{An ablation study investigating the efficacy of Brainomaly in transductive and inductive learning settings. The study assesses the performance, measured by the AUC metric for Alzheimer's Disease and Headache detection datasets. The table demonstrates that Brainomaly's performance fluctuates across different datasets, as expected in data-driven methods. However, it reveals a statistically insignificant difference ({\em p}-value $>$ 0.005) in Brainomaly's performance between the transductive and inductive learning settings. The models selected by the $\boldsymbol{AUC}_p$ metric consistently produce higher detection rates across datasets and operating settings.}
    \centering
    \begin{tabular}{l l c c c c}
        \hline
        \parbox[t]{8mm}{\multirow{8}{*}{\scriptsize\rotatebox[origin=c]{90}{\textbf{FID} selected Models}}} & \multicolumn{5}{c}{\textbf{Alzheimer's Disease Dataset}} \\
        \cline{2-6}
        & {\em Settings} & {\em AD DS1} & {\em AD DS2} & Avg. & {\em p}-value \\
        \cline{2-6}
        & Transductive & 0.6180 & 0.6771 & 0.6476 & \multirow{2}{*}{0.8714} \\
        & Inductive & 0.6389 & 0.6453 & 0.6421 \\
        \cline{2-6}
        & \multicolumn{5}{c}{\textbf{Headache Dataset}} \\
        \cline{2-6}
        & {\em Settings} & {\em HEAD DS1} & {\em HEAD DS2} & Avg. & {\em p}-value \\
        \cline{2-6}
        & Transductive & 0.8807 & 0.9120 & 0.8964 & \multirow{2}{*}{0.5832} \\
        & Inductive & 0.9002 & 0.8589 & 0.8796 \\
        \hline
        \hline
        \parbox[t]{8mm}{\multirow{8}{*}{\scriptsize\rotatebox[origin=c]{90}{ $\boldsymbol{AUC}_p$  selected models}}} & \multicolumn{5}{c}{\textbf{Alzheimer's Disease Dataset}} \\
        \cline{2-6}
        & {\em Settings} & {\em AD DS1} & {\em AD DS2} & Avg. & {\em p}-value \\
        \cline{2-6}
        & Transductive & 0.6526 & 0.6825 & 0.6676 & \multirow{2}{*}{0.5553} \\
        & Inductive & 0.6452 & 0.6648 & 0.6550 \\
        \cline{2-6}
        & \multicolumn{5}{c}{\textbf{Headache Dataset}} \\
        \cline{2-6}
        & {\em Settings} & {\em HEAD DS1} & {\em HEAD DS2} & Avg. & {\em p}-value \\
        \cline{2-6}
        & Transductive & 0.9182 & 0.8633 & 0.8908 & \multirow{2}{*}{0.8726} \\
        & Inductive & 0.9041 & 0.8878 & 0.8960 \\
        \hline
    \end{tabular}
    \label{tab_aucp:transductive_vs_inductive}
\end{table*}

\begin{figure*}[!htp]
    \centering
    \includegraphics[width=0.7\linewidth, height=0.23\textheight]{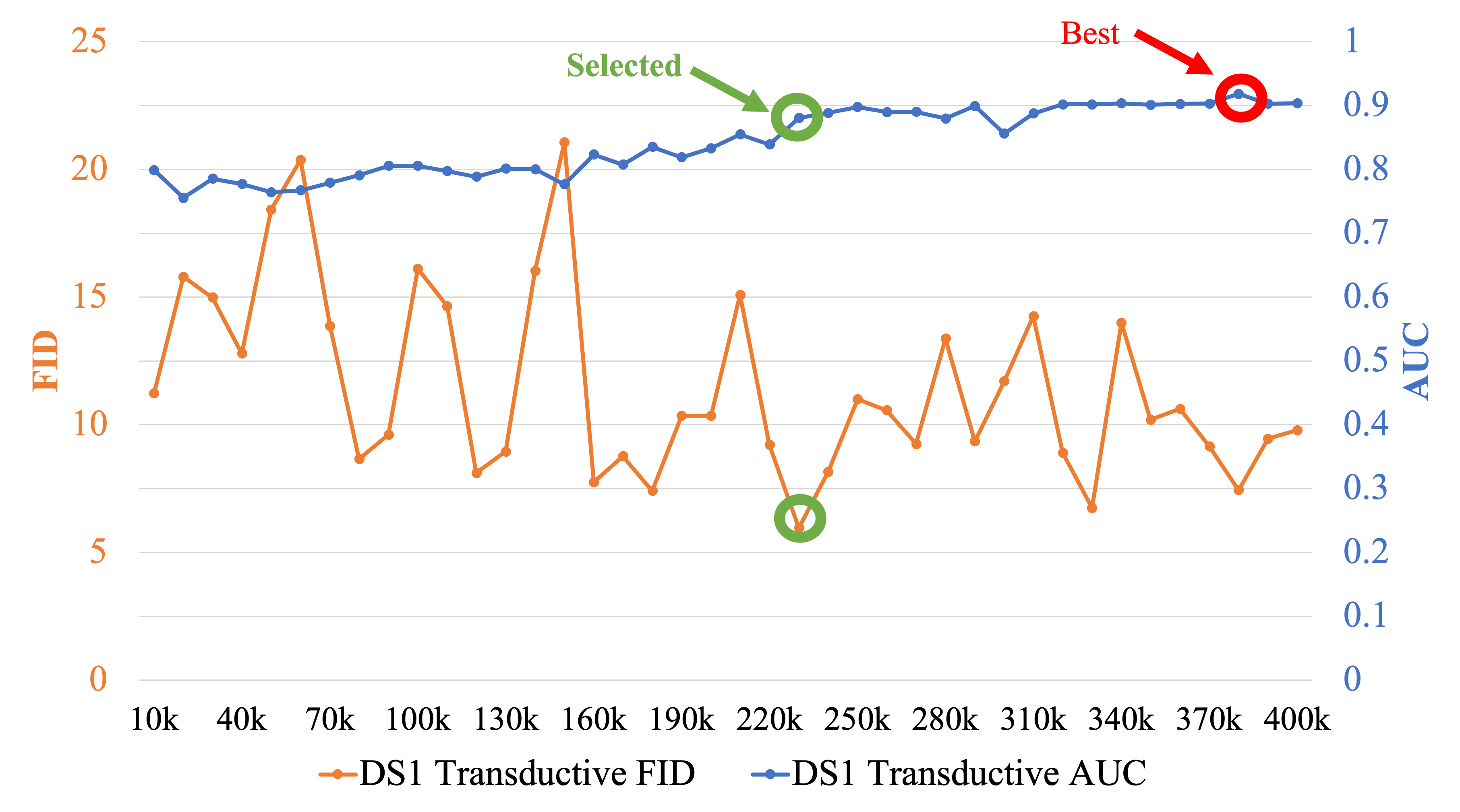}
    \caption{Correlation of FID scores and actual abnormality detection performance of Brainomaly's models in AUC on headache ({\em HEAD DS1}) dataset. The X-axis shows the training iteration numbers when the models were saved. In \cite{rahman2022healthygan} and \cite{siddiquee2023brainomaly}, FID is used for selecting the inference model. FID selects a generative model based on how realistic images (lowest FID score) the model can produce. In our case, FID selected models for inference that could produce the most realistic images of normal or healthy anatomies. However, our objective is not to generate realistic images of healthy anatomies. Rather, we aim to achieve the best abnormality detection performance. Since FID does not select an inference model directly based on our abnormality detection objective, it often selects a sub-optimal model for abnormality detection. More in~\sectionautorefname~\ref{subsec_aucp:results_model_selection}.}
    \label{fig_aucp:fid_vs_aucp}
\end{figure*}

\tableautorefname~\ref{tab_aucp:transductive_vs_inductive} summarizes the performance of Brainomaly's models when selected using FID and $\boldsymbol{AUC}_p$ metrics for Alzheimer's disease and headache detection in both transductive and inductive operating settings. For Alzheimer's disease and headache detection, it is observed  the model's performance is slightly influenced by the random split of the unannotated set of mixed brain MRIs used during the training. For example, the AUCs are slightly higher for {\em AD DS2} compared to {\em AD DS1} and similarly for {\em HEAD DS1} compared to {\em HEAD DS2}. However, the overall performance of Brainomaly's models, irrespective of model selection metrics, for AD and headache detection is statistically the same ($p$-value $>$ 0.005) in both transductive and inductive operating settings. Please note that the model selected by $\boldsymbol{AUC}_p$ consistently outperforms the model selected by FID. These results show that Brainomaly generalizes well on unseen test data, and using $\boldsymbol{AUC}_p$ metric can help to select a higher performing model without any additional labeling cost.

\begin{table}[!tp]
    \caption{Pearson correlation of FID and the proposed $\boldsymbol{AUC}_p$ metric with the actual AUC metric. $\boldsymbol{AUC}_p$ has a stronger correlation with the actual AUC. Therefore, in a situation like ours, $\boldsymbol{AUC}_p$ is a superior measure to FID for choosing the inference model. The correlation numbers are in absolute terms. Numbers in \textbf{boldface} indicate the best correlation. The metrics are calculated in a transductive setting.}
    \centering
    \begin{tabular}{l c c}
        \hline
        \multicolumn{3}{c}{\textbf{Alzheimer's Disease Dataset}} \\
        \hline
         Metric & {\em AD DS1} & {\em AD DS2} \\
        \hline
        FID & 0.5701 & 0.4773 \\
        \textbf{$\boldsymbol{AUC}_p$ (Our)} & \textbf{0.9583} & \textbf{0.9656} \\
        \hline
        \hline
        \multicolumn{3}{c}{\textbf{Headache Dataset}} \\
        \hline
        Metric & {\em HEAD DS1} & {\em HEAD DS2} \\
        \hline
        FID & 0.5227 & 0.3187 \\
        \textbf{$\boldsymbol{AUC}_p$ (Our)} & \textbf{0.9528} & \textbf{0.5986} \\
        \hline
    \end{tabular}
    \label{tab_aucp:model_selection_pearson}
\end{table}

\begin{table}[!tp]
    \caption{Performance comparison of the final models selected by FID and the proposed $\boldsymbol{AUC}_p$ metrics. As shown in~\tableautorefname~\ref{tab_aucp:model_selection_pearson}, $\boldsymbol{AUC}_p$ strongly correlates with the actual AUC metric. This table shows that the model selected by the $\boldsymbol{AUC}_p$ outperforms the model selected by FID. Numbers in \textbf{boldface} indicate the best AUCs.}
    \centering
    \begin{tabular}{l l c c}
        \hline
        \multicolumn{4}{c}{\textbf{Alzheimer's Disease Dataset}} \\
        \hline
        Setting & Metric & {\em AD DS1} & {\em AD DS2} \\
        \hline
        \multirow{2}{*}{Transductive} & FID & 0.618 & 0.6771 \\
         & \textbf{$\boldsymbol{AUC}_p$ (Our)} & \textbf{0.6526} & \textbf{0.6825} \\
        \hline
        \multirow{2}{*}{Inductive} & FID & 0.6389 & 0.6453 \\
         & \textbf{$\boldsymbol{AUC}_p$ (Our)} & \textbf{0.6452} & \textbf{0.6648} \\
        \hline
        \hline
        \multicolumn{4}{c}{\textbf{Headache Dataset}} \\
        \hline
        Setting & Metric & {\em HEAD DS1} & {\em HEAD DS2} \\
        \hline
        \multirow{2}{*}{Transductive} & FID & 0.8807 & \textbf{0.9120} \\
         & \textbf{$\boldsymbol{AUC}_p$ (Our)} & \textbf{0.9182} & 0.8633 \\
        \hline
        \multirow{2}{*}{Inductive} & FID & 0.9002 & 0.8589 \\
         & \textbf{$\boldsymbol{AUC}_p$ (Our)} & \textbf{0.9041} & \textbf{0.8878} \\
        \hline
    \end{tabular}
    \label{tab_aucp:model_selection_auc}
\end{table}

\subsubsection{Model Selection for Inference: FID vs. $\boldsymbol{AUC}_p$ }
\label{subsec_aucp:results_model_selection}

In \cite{rahman2022healthygan} and \cite{siddiquee2023brainomaly}, FID is used for selecting the inference model. FID selects a generative model based on how realistic the images are (lowest FID score), here, realistic images of normal anatomies. However, the objective is not to generate realistic images of healthy anatomies but to achieve the best abnormality detection performance. Since FID does not select an inference model directly based on the abnormality detection objective, we find that it often selects a sub-optimal model for abnormality detection.~\figureautorefname~\ref{fig_aucp:fid_vs_aucp} shows such phenomenon. In this section, we formally compare (\tableautorefname~\ref{tab_aucp:model_selection_pearson}) the correlation of models' FID and $\boldsymbol{AUC}_p$ scores with models' actual abnormality detection performance. In~\tableautorefname~\ref{tab_aucp:model_selection_auc}, we also contrast the abnormality detection performances among the best models selected by each competing metric.

In~\tableautorefname~\ref{tab_aucp:model_selection_pearson}, it is shown  that $\boldsymbol{AUC}_p$ score in~\sectionautorefname~\ref{sec_aucp:method} has a stronger correlation with the actual (when all the annotations are available) AUC scores. Therefore, the proposed $\boldsymbol{AUC}_p$ metric renders itself a better metric than FID for selecting the model for inference. To further validate, we have provided the AUC scores obtained by the best models according to FID and the $\boldsymbol{AUC}_p$ scores in both transductive and inductive learning settings for each dataset in~\tableautorefname~\ref{tab_aucp:model_selection_auc}. The models selected by  $\boldsymbol{AUC}_p$ metric dominate in detection performance over the models selected by FID.

\begin{table*}[!ht]
\caption{Performance comparison of the final (last) models versus the models selected by the proposed $\boldsymbol{AUC}_{p}$ metric across multiple image reconstruction, feature-reconstruction and self-supervised methods on seven datasets. The “Last” column reports the $\boldsymbol{AUC}$ achieved by the model at the end of training, while the “AUCp” column indicates the AUC from the epoch that achieved the highest $\boldsymbol{AUC}_{p}$ score. The results show that $\boldsymbol{AUC}_{p}$-based selection generally improves performance across methods, with only a few exceptions, underscoring both the effectiveness of $\boldsymbol{AUC}_{p}$ under specific conditions.}
\centering
\resizebox{\textwidth}{!}{
\begin{tabular}{l l*{14}{c}}
\toprule
\multirow{2}{*}{{Category}} &  & \multicolumn{2}{c}{RSNA} & \multicolumn{2}{c}{VIN} & \multicolumn{2}{c}{ISIC} & \multicolumn{2}{c}{LAG} & \multicolumn{2}{c}{C16} & \multicolumn{2}{c}{Brain} & \multicolumn{2}{c}{Brats21} \\
\cmidrule(lr){3-4}\cmidrule(lr){5-6}\cmidrule(lr){7-8}\cmidrule(lr){9-10}\cmidrule(lr){11-12}\cmidrule(lr){13-14}\cmidrule(lr){15-16}
& Method & Last & AUCp & Last & AUCp & Last & AUCp & Last & AUCp & Last & AUCp & Last & AUCp & Last & AUCp \\
\midrule
\multirow{7}{*}{{Self-sup}} 
 & CutPaste-Normal & 0.7253 & \textbf{0.7314} & 0.5705 & \textbf{0.6064} & 0.7442 & \textbf{0.7704} & 0.6052 & \textbf{0.6391} & 0.4289 & \textbf{0.5792} & 0.9161 & \textbf{0.9167} & 0.5190 & \textbf{0.5576} \\
 & CutPaste-3way   & 0.7671 & \textbf{0.7422} & 0.6123 & \textbf{0.6157} & 0.7231 & \textbf{0.7606} & 0.6095 & \textbf{0.6191} & 0.4279 & \textbf{0.5413} & 0.7642 & \textbf{0.7717} & 0.5385 & \textbf{0.7419} \\
 & Anapaste        & 0.8882 & \textbf{0.8862} & 0.7107 & \textbf{0.7143} & 0.8180 & \textbf{0.8195} & 0.7304 & \textbf{0.7364} & 0.5574 & \textbf{0.6214} & 0.9435 & \textbf{0.9436} & 0.7469 & \textbf{0.7497} \\
 & CutPaste        & 0.6117 & \textbf{0.7320} & 0.5459 & \textbf{0.6722} & 0.5442 & \textbf{0.6806} & 0.4892 & \textbf{0.6890} & 0.6418 & \textbf{0.7058} & 0.2208 & \textbf{0.8277} & 0.4650 & \textbf{0.6013} \\
 & FPI             & 0.4937 & \textbf{0.7067} & 0.4746 & \textbf{0.6914} & 0.6749 & \textbf{0.7524} & 0.5995 & \textbf{0.6994} & 0.6301 & \textbf{0.6468} & 0.2932 & \textbf{0.8881} & 0.5142 & \textbf{0.7645} \\
 & FPI-Poisson     & 0.8124 & \textbf{0.8204} & 0.5670 & \textbf{0.7612} & 0.4655 & \textbf{0.5737} & 0.3934 & \textbf{0.6701} & 0.3719 & \textbf{0.6546} & 0.2425 & \textbf{0.7809} & 0.8420 & \textbf{0.8736} \\
 & NSA             & 0.8054 & \textbf{0.8673} & 0.4999 & \textbf{0.7622} & 0.6439 & \textbf{0.7114} & 0.6180 & \textbf{0.7540} & 0.6083 & \textbf{0.6604} & 0.7853 & \textbf{0.9186} & \textbf{0.8734} & 0.8123 \\

\midrule
\multirow{6}{*}{{Image-rec}} 
 & AEU            & 0.7835 & \textbf{0.8763} & 0.6833 & \textbf{0.7334} & 0.6926 & \textbf{0.7098} & 0.8173 & \textbf{0.8442} & 0.5991 & \textbf{0.6209} & 0.9042 & \textbf{0.9380} & 0.8153 & \textbf{0.8375} \\
 & MemAE          & 0.6572 & \textbf{0.6963} & 0.5465 & \textbf{0.5681} & 0.7014 & \textbf{0.7410} & 0.7583 & \textbf{0.7995} & 0.3665 & \textbf{0.3807} & 0.6030 & \textbf{0.8435} & 0.8095 & \textbf{0.8217} \\
 & AE             & \textbf{0.6702} & 0.6629 & 0.4696 & \textbf{0.5061} & 0.7301 & \textbf{0.7494} & 0.8214 & \textbf{0.8308} & 0.3645 & \textbf{0.3682} & 0.8515 & \textbf{0.8655} & \textbf{0.8248} & \textbf{0.8248} \\
 & AE-L1          & 0.6689 & \textbf{0.6837} & 0.5596 & \textbf{0.5623} & 0.7482 & \textbf{0.7534} & 0.7513 & \textbf{0.7728} & 0.3404 & \textbf{0.3433} & \textbf{0.8330} & \textbf{0.8330} & 0.7851 & \textbf{0.8221} \\
 & AE-SSIM        & 0.6426 & \textbf{0.6472} & 0.4477 & \textbf{0.4577} & 0.7688 & \textbf{0.7704} & \textbf{0.6802} & \textbf{0.6802} & \textbf{0.3949} & \textbf{0.3949} & 0.9217 & \textbf{0.9355} & 0.8010 & \textbf{0.7985} \\
 & AE-Perceptual  & 0.8649 & \textbf{0.8669} & 0.6988 & \textbf{0.7027} & 0.6904 & 0.6877 & 0.8549 & \textbf{0.8562} & 0.7580 & \textbf{0.7586} & \textbf{0.9558} & \textbf{0.9558} & 0.8568 & \textbf{0.8572} \\
\midrule
\multirow{2}{*}{{Feature-rec}} 
 & FAE-SSIM       & \textbf{0.9053} & \textbf{0.9053} & 0.6671 & \textbf{0.6704} & \textbf{0.7875} & \textbf{0.7875} & \textbf{0.8342} & \textbf{0.8342} & 0.4857 & \textbf{0.4958} & \textbf{0.9205} & \textbf{0.9205} & \textbf{0.7620} & \textbf{0.7620} \\
 & FAE-MSE        & \textbf{0.7423} & \textbf{0.7423} & \textbf{0.5309} & \textbf{0.5309} & \textbf{0.7473} & \textbf{0.7473} & \textbf{0.7349} & \textbf{0.7349} & \textbf{0.4473} & \textbf{0.4473} & 0.8273 & \textbf{0.8301} & 0.7870 & \textbf{0.7918} \\

\bottomrule
\end{tabular}
}
\label{tab:results_all_grouped_vertical}
\end{table*}

\begin{table*}[t]
\centering
\caption{Performance comparison of reconstruction-based anomaly detection methods using standard reconstruction loss versus the proposed AUCp-based inference model selection across multiple medical imaging datasets. The \textit{Rec\_Loss} column reports the image-level AUC obtained when the inference model is selected using reconstruction loss, while the \textit{AUCp} column indicates the AUC achieved by the model selected based on the proposed AUCp criterion.}
\label{tab:aucp_reconstruction_comparison}
\resizebox{\textwidth}{!}{
\begin{tabular}{lcccccccccccccc}
\toprule
\textbf{Method} 
& \multicolumn{2}{c}{\textbf{RSNA}} 
& \multicolumn{2}{c}{\textbf{VIN}} 
& \multicolumn{2}{c}{\textbf{ISIC}} 
& \multicolumn{2}{c}{\textbf{LAG}} 
& \multicolumn{2}{c}{\textbf{C16}} 
& \multicolumn{2}{c}{\textbf{Brain}} 
& \multicolumn{2}{c}{\textbf{Brats21}} \\
\cmidrule(lr){2-3} \cmidrule(lr){4-5} \cmidrule(lr){6-7} \cmidrule(lr){8-9}
\cmidrule(lr){10-11} \cmidrule(lr){12-13} \cmidrule(lr){14-15}
& Rec\_Loss & AUCp
& Rec\_Loss & AUCp
& Rec\_Loss & AUCp
& Rec\_Loss & AUCp
& Rec\_Loss & AUCp
& Rec\_Loss & AUCp
& Rec\_Loss & AUCp \\
\midrule
AEU          
& 0.8692 & \textbf{0.8763}
& 0.7141 & \textbf{0.7334}
& 0.7011 & \textbf{0.7098}
& 0.8365 & \textbf{0.8442}
& 0.5642 & \textbf{0.6209}
& \textbf{0.9416} & 0.9380
& 0.8124 & \textbf{0.8375} \\

MemAE        
& 0.6751 & \textbf{0.6963}
& 0.5484 & \textbf{0.5681}
& 0.7353 & \textbf{0.7410}
& 0.7789 & \textbf{0.7995}
& 0.3416 & \textbf{0.3807}
& 0.8296 & \textbf{0.8435}
& 0.7371 & \textbf{0.8217} \\

AE           
& 0.6421 & \textbf{0.6629}
& \textbf{0.5603} & 0.5061
& \textbf{0.7500} & 0.7494
& 0.8103 & \textbf{0.8308}
& 0.3645 & \textbf{0.3682}
& 0.8614 & \textbf{0.8655}
& \textbf{0.8248} & 0.8248 \\

AE-L1        
& 0.6648 & \textbf{0.6837}
& 0.5614 & \textbf{0.5623}
& 0.7481 & \textbf{0.7534}
& 0.7513 & \textbf{0.7728}
& 0.3403 & \textbf{0.3433}
& 0.8184 & \textbf{0.8330}
& 0.8088 & \textbf{0.8221} \\

AE-SSIM      
& \textbf{0.6505} & 0.6472
& 0.4574 & \textbf{0.4577}
& 0.7687 & \textbf{0.7704}
& \textbf{0.6802} & \textbf{0.6802}
& 0.3948 & \textbf{0.3949}
& 0.9217 & \textbf{0.9355}
& \textbf{0.8010} & 0.7985 \\

AE-Perceptual
& 0.8645 & \textbf{0.8669}
& \textbf{0.7090} & 0.7027
& \textbf{0.6904} & 0.6877
& 0.8549 & \textbf{0.8562}
& \textbf{0.7586} & 0.7586
& \textbf{0.9558} & 0.9558
& 0.8568 & \textbf{0.8572} \\
\bottomrule
\end{tabular}
}
\end{table*}

\subsection{Self-supervised \& Reconstruction Based Method}
\label{subsec:self_supervised_results}
The $AUCp$ method demonstrates significant advancements in the detection of various medical conditions by enhancing the performance of self-supervised, image-reconstruction, and feature reconstruction based abnormality detection methods. These improvements are observed across multiple datasets which is presented in Tab. \ref{tab:results_all_grouped_vertical} and \ref{tab:aucp_reconstruction_comparison}, leading to more reliable and accurate disease detection, which is critical for clinical applications.
\subsubsection{Chest/Pulmonary Abnormality Detection}
Chest and pul- monary abnormality detection, such as identifying lung opaci- ties, is crucial for diagnosing diseases like pneumonia and other pulmonary pathologies. The RSNA dataset, which contains X- rays showing lung opacities, has seen significant improvements when the $\boldsymbol{AUC}_p$ method was applied. For example, CutPaste’s AUC increases from 61.17\% to 73.20\%, showing a notable improvement in its ability to detect pulmonary anomalies. Similarly, FPI’s AUC rise from 49.37\% to 70.67\%, while NSA has seen an increase from 80.54\% to 86.73\%. The VIN dataset, which also focuses on chest-related anomalies like viral infections, demonstrate similar improvements. Here, CutPaste’s AUC increase from 54.59\% to 67.22\%, FPI’s from 47.46\% to 69.14\%, and NSA’s performance jumps from 49.99\% to 76.22\%. Turning to image-reconstruction baselines, RSNA shows clear improvements for AEU (78.35\% to 87.63\%), MemAE (65.72\% to 69.63\%), AE-L1 (66.89\% to 68.37\%), and AE-SSIM (64.26\% to 64.72\%), with AE-Perceptual nudging up (86.49\% to 86.69\%) and vanilla AE slightly down (67.02\% to 66.29\%). On VIN, AEU (68.33\% to 73.34\%), AE-Perceptual (69.88\% to 70.27\%), MemAE (54.65\% to 56.81\%), AE (46.96\% to 50.61\%), AE-L1 (55.96\% to 56.23\%), and AE-SSIM (44.77\% to 45.77\%) all benefit to varying degrees. Feature-reconstruction methods are largely insensitive to the selection rule: FAE-SSIM stays at 90.53\% (the best absolute score on RSNA), and FAE-MSE remains at 74.23\%; on VIN, FAE-SSIM moves marginally (66.71\% to 67.04\%) while FAE-MSE is unchanged (53.09\%). Overall, $\boldsymbol{AUC}_p$ improves most self-supervised and image-reconstruction models, yielding the strongest VIN gains for NSA/FPI-Poisson and the largest RSNA gain for AEU, while feature-space reconstruction remains near a ceiling on RSNA.

We also evaluate the effectiveness of $AUC_p$ as an inference criterion relative to standard reconstruction loss for chest and pulmonary abnormality detection. As summarized in Tab. \ref{tab:aucp_reconstruction_comparison}, $AUC_p$-based inference consistently matches or outperforms reconstruction loss across most image-reconstruction methods on both RSNA and VIN datasets. On RSNA, $AUC_p$ improves inference performance for AEU, MemAE, AE-L1, and AE-SSIM, while maintaining comparable performance for AE-Perceptual and exhibiting only minor degradation for vanilla AE. A similar pattern is observed on VIN, where $AUC_p$ yields higher AUC for AEU, MemAE, AE, AE-L1, and AE-SSIM, with AE-Perceptual remaining largely stable. These findings indicate that, even in chest X-ray data where reconstruction error can be effective, $AUC_p$ provides a more reliable and consistent inference signal than reconstruction loss, particularly for models sensitive to background anatomical variability.

\subsubsection{Brain Tumor Detection}
Brain tumor detection is an essential area in medical imaging, as early detection can significantly improve patient outcomes. The Brain dataset, which focuses on detecting brain tumors, demonstrates the significant impact of $AUC_p$ on model performance. For instance, FPI’s $AUC$ rises from 29.32\% to 88.81\%, a clear indication of the method’s ability to select an inference model for better tumor detection accuracy. Similarly, NSA improves from 78.53\% to 91.86\%, marking a strong boost in its ability to identify brain abnormalities. Additional self-supervised methods also benefit: CutPaste increases from 22.08\% to 82.77\%, FPI-Poisson from 24.25\% to 78.09\%, CutPaste-3way from 76.42\% to 77.17\%, CutPaste-Normal from 91.61\% to 91.67\%, and Anapaste from 94.35\% to 94.36\%. Reconstruction baselines follow the same trend: among image-reconstruction methods, AEU improves from 90.42\% to 93.80\%, MemAE from 60.30\% to 84.35\%, AE from 85.15\% to 86.55\%, AE-L1 remains at 83.30\%, AE-SSIM increases from 92.17\% to 93.55\%, and AE-Perceptual remains at 95.58\% (the best overall on Brain). In feature-reconstruction, FAE-SSIM remains at 92.05\%, while FAE-MSE increases from 82.73\% to 83.01\%.

The Brats21 dataset, focused specifically on brain tumor detection, also benefits from $AUC_p$. CutPaste’s $AUC$ rises from 46.50\% to 60.13\%, while FPI-Poisson increases from 84.20\% to 87.36\%. The remaining self-supervised methods show consistent patterns: FPI rises from 51.42\% to 76.45\%, CutPaste-3way from 53.85\% to 74.19\%, CutPaste-Normal from 51.90\% to 55.76\%, Anapaste from 74.69\% to 74.97\%, and NSA decreases from 87.34\% to 81.23\%. Reconstruction methods also improve in most cases: AEU rises from 81.53\% to 83.75\%, MemAE from 80.95\% to 82.17\%, AE-L1 from 78.51\% to 82.21\%, and AE-Perceptual from 85.68\% to 85.72\%, while AE remains at 82.48\% and AE-SSIM changes from 80.10\% to 79.85\%. Feature-reconstruction methods show small changes, with FAE-SSIM remaining at 76.20\% and FAE-MSE increasing from 78.70\% to 79.18\%. Despite some smaller gains in Brats21, these improvements are still meaningful, particularly for detecting complex brain tumors where precision is critical for treatment planning.

We further analyze the role of $AUC_p$ as an inference criterion for image reconstruction-based methods by comparing it directly against standard reconstruction loss. As shown in Tab. \ref{tab:aucp_reconstruction_comparison}, $AUC_p$-based inference consistently yields equal or higher $AUC$ across most image-reconstruction models on both Brain and Brats21 datasets. On the Brain dataset, $AUC_p$ improves performance for AEU, MemAE, AE, AE-L1, and AE-SSIM, while maintaining the already strong performance of AE-Perceptual. Similar trends are observed on Brats21, where $AUC_p$ leads to notable gains for AEU, MemAE, AE-L1, and AE-Perceptual, with only marginal changes for AE and AE-SSIM. These results indicate that, even when the training epoch is fixed, reconstruction loss alone may not provide a reliable anomaly scoring signal for complex brain tumor data, whereas $AUC_p$ offers a more robust and consistent inference strategy.

\subsubsection{Retinal Fundus Detection}
Retinal fundus imaging is critical for detecting abnormalities in the retina, such as diabetic retinopathy, macular degeneration, and other eye conditions. The LAG dataset, which contains retinal fundus images, demonstrates significant improvements in disease detection when the $AUC_p$ method is applied. Specifically, CutPaste’s $AUC$ increases from 48.92\% to 68.90\%, indicating a substantial improvement in detecting retinal anomalies. FPI also shows improvement, with its $AUC$ rising from 59.95\% to 69.94\%. NSA, which starts at 61.80\%, reaches 75.40\%, reflecting a notable enhancement in the ability to detect subtle retinal issues. Additional self-supervised methods follow the same trend: FPI-Poisson improves from 39.34\% to 67.01\%, CutPaste-Normal from 60.52\% to 63.91\%, CutPaste-3way from 60.95\% to 61.91\%, and Anapaste from 73.04\% to 73.64\%. Reconstruction baselines also benefit: among image-reconstruction methods, AEU increases from 81.73\% to 84.42\%, MemAE from 75.83\% to 79.95\%, AE from 82.14\% to 83.08\%, AE-L1 from 75.13\% to 77.28\%, AE-SSIM remains at 68.02\%, and AE-Perceptual rises from 85.49\% to 85.62\%. Feature-reconstruction methods are unchanged on LAG, with FAE-SSIM at 83.42\% and FAE-MSE at 73.49\%. Overall, these results indicate that $AUC_p$ consistently selects stronger inference checkpoints for both self-supervised and reconstruction-based paradigms on retinal fundus imaging.

Tab \ref{tab:aucp_reconstruction_comparison} further shows that $AUC_p$ serves as a more reliable inference model selection method than reconstruction loss for retinal fundus imaging. On the LAG dataset, $AUC_p$ improves AUC for most image-reconstruction methods, including AEU, MemAE, AE, and AE-L1, while preserving the strong performance of AE-Perceptual and showing minimal change for AE-SSIM. These results indicate that, in relatively structured retinal imagery, $AUC_p$ provides consistent yet stable gains without introducing sensitivity to reconstruction artifacts.


\subsubsection{Lymph Node Metastasis Detection}
The Camelyon16 dataset, which focuses on images related to lymph node metastasis, is also benefited significantly from the $AUC_p$ method. In this dataset, CutPaste’s $AUC$ improves from 64.18\% to 70.58\%, showing a clear boost in the selected inference model’s ability to detect metastatic spread. FPI’s $AUC$ increases from 63.01\% to 70.46\%, and FPI-Poisson rises from 54.68\% to 71.31\%. Similarly, NSA’s $AUC$ improves from 60.83\% to 66.04\%. The remaining self-supervised variants also gain: CutPaste-Normal improves from 42.89\% to 57.92\%, CutPaste-3way from 42.79\% to 54.13\%, and Anapaste from 55.74\% to 62.14\%. Reconstruction-based baselines follow the same overall trend. Among image-reconstruction methods, AEU improves from 59.91\% to 62.09\%, MemAE from 36.65\% to 38.07\%, AE from 36.45\% to 36.82\%, AE-L1 from 34.04\% to 34.33\%, AE-SSIM remains at 39.49\%, and AE-Perceptual increases from 75.80\% to 75.86\%. In feature reconstruction, FAE-SSIM moves from 48.57\% to 49.58\%, while FAE-MSE remains at 44.73\%. Collectively, these results indicate that $AUC_p$ refining inference-time checkpoint selection for both self-supervised and reconstruction-based paradigms on Camelyon16, improving reliability for metastasis detection.

We further compare reconstruction loss and $AUC_p$ as inference criteria for lymph node metastasis detection on Camelyon16. As shown in Tab. \ref{tab:aucp_reconstruction_comparison}, $AUC_p$-based inference yields modest but consistent improvements over reconstruction loss for most image-reconstruction methods, including AEU, MemAE, AE, and AE-L1, while maintaining the strong performance of AE-Perceptual and showing negligible change for AE-SSIM. These results suggest that, in histopathology images with high structural variability, $AUC_p$ provides a more stable inference signal than reconstruction error alone, contributing to more reliable metastasis detection.

\subsubsection{Skin Lesion Detection}
Skin lesion detection, particularly for skin cancers such as melanoma, is an area where accurate abnormality detection is critical for improving patient outcomes. The ISIC dataset, which focuses on detecting skin lesions, shows notable improvements with the $AUC_p$ method. CutPaste’s $AUC$ improves from 54.42\% to 68.06\%, while FPI’s performance rises from 67.49\% to 75.24\%. FPI-Poisson’s $AUC$, though initially lower, increases from 46.55\% to 57.37\%, and NSA’s $AUC$ improves from 64.39\% to 71.14\%. Additional self-supervised variants show consistent behavior: CutPaste-Normal increases from 74.42\% to 77.04\%, CutPaste-3way from 72.31\% to 76.06\%, and Anapaste from 81.80\% to 81.95\%. Reconstruction baselines also benefit. Among image-reconstruction methods, AEU changes from 69.26\% to 70.98\%, MemAE from 70.14\% to 74.10\%, AE from 73.01\% to 74.94\%, AE-L1 from 74.82\% to 75.34\%, and AE-SSIM from 76.88\% to 77.04\%, while AE-Perceptual shows a slight decrease from 69.04\% to 68.77\%. In feature-reconstruction, FAE-SSIM remains at 78.75\% and FAE-MSE remains at 74.73\%. These results indicate that $AUC_p$ consistently selects stronger inference checkpoints for both self-supervised and reconstruction-based paradigms on ISIC, which is particularly important for early, accurate skin-cancer detection.

To further analyze inference behavior on skin lesion detection, we compare reconstruction loss with $AUC_p$ as the scoring criterion for reconstruction-based methods on ISIC. As summarized in Tab. \ref{tab:aucp_reconstruction_comparison}, $AUC_p$-based inference generally improves or preserves AUC across most image-reconstruction models, including AEU, MemAE, AE, AE-L1, and AE-SSIM, while the performance of AE-Perceptual remains comparable.

\subsection{Analysis of AUCp: Cost and Correlation with AUC}

A central goal of this work is to use $\boldsymbol{AUC}_p$ as a \emph{label–free} criterion for choosing the inference checkpoint across training epochs. This section represents practical aspects that determine whether $\boldsymbol{AUC}_p$ is viable in real deployments: its computational cost and its correlation with the $\boldsymbol{AUC}$.

\subsubsection{AUCp Runtime Cost Analysis}
The computational cost of calculating AUCp is essentially the same as that of a standard AUC computation. Given $n$ samples with anomaly scores ${s_i}$ and binary labels ${y_i \in {0,1}}$, the ROC curve is constructed by thresholding at each unique score. Since the ROC curve only changes at observed score values, the number of thresholds is at most the number of unique scores, i.e., $|T| \leq n$. The main steps are:
\begin{itemize}
\item Extracting unique thresholds: $O(n)$
\item Sorting scores in descending order: $O(n \log n)$
\item Computing cumulative sums and the ROC integral: $O(n)$
\end{itemize}
Thus, the overall time complexity is $O(n \log n)$ with space complexity $O(n)$, which is identical to conventional AUC computation. Importantly, AUCp introduces no additional asymptotic complexity beyond standard AUC, and in practice the sorting step remains negligible compared to the cost of model inference, even for large-scale or high-resolution datasets.

\subsubsection{Calculating AUC from AUCp}
Let $s(\cdot)$ denote the anomaly score function, where higher values indicate greater abnormality. 
The true area under the ROC curve (AUC) is defined as
\begin{equation}
\label{eq:true_auc}
\text{AUC} \;=\; \Pr\{ s(A) > s(N_{\text{te}}) \},
\end{equation}
where $A$ represents abnormal samples and $N_{\text{te}}$ denotes true normal test samples.  

In practice, the pseudo-AUC (AUCp) compares training normals $N_{\text{tr}}$ against the pseudo-positive test set, which is a mixture of abnormals and normals:
\begin{equation}
\label{eq:aucp_def_set}
\text{AUC}_{p} \;=\; \Pr\{ s(U) > s(N_{\text{tr}}) \}, \quad U \sim (1-\rho)A \cup \rho N_{\text{te}},
\end{equation}
where $\rho = \tfrac{|N_{\text{te}}|}{|N_{\text{te}}|+|A|}$ is the contamination fraction of true normals in the pseudo-positive set.  

Applying the law of total probability, AUCp decomposes as
\begin{equation}
\label{eq:aucp_mix}
\text{AUC}_{p} \;=\; (1-\rho)\,\Pr\{ s(A) > s(N_{\text{tr}}) \} \;+\; \rho\,\Pr\{ s(N_{\text{te}}) > s(N_{\text{tr}}) \}.
\end{equation}

If the training and test normals are identically distributed (i.e., no covariate shift), then
\begin{equation}
\Pr\{ s(N_{\text{te}}) > s(N_{\text{tr}}) \} = 0.5,
\end{equation}
and furthermore
\begin{equation}
\Pr\{ s(A) > s(N_{\text{tr}}) \} = \Pr\{ s(A) > s(N_{\text{te}}) \} = \text{AUC}.
\end{equation}

Therefore, under the no-shift assumption, we obtain a closed-form relationship between the pseudo-AUC and the true AUC:
\begin{equation}
\label{eq:aucp_to_auc}
\text{AUC}_{p} \;=\; (1-\rho)\,\text{AUC} \;+\; \rho \cdot 0.5.
\end{equation}

Rearranging gives an unbiased estimator of the true AUC from AUCp:
\begin{equation}
\label{eq:auc_from_aucp}
\text{AUC} \;\approx\; \frac{\text{AUC}_{p} - 0.5\rho}{1 - \rho}.
\end{equation}

This expression demonstrates that AUCp shrinks linearly toward 0.5 as the contamination fraction $\rho$ increases, but can be corrected if $\rho$ is estimated using mixture proportion estimation (MPE) methods developed in the positive--unlabeled learning literature, such as Kernel Mean Embedding MPE \cite{ramaswamy2016mixture}, AlphaMax \cite{jain2016estimating}, or TIcE \cite{bekker2018estimating}.

\subsubsection{Correlation between $\boldsymbol{AUC}_p$ and $\boldsymbol{AUC}$}
The correlation between $AUC_p$ and $AUC$ across different self-supervised abnormality detection methods (CutPaste, FPI, FPI Poisson, and NSA) is analyzed, as shown in the Tab. \ref{tab:pearson_correlation}. We have these methods as they showed the best improvement after applying the $AUC_p$ method, and for each dataset, the results demonstrate a strong positive correlation between $\boldsymbol{AUC}_p$ and $\boldsymbol{AUC}$ for most datasets, indicating that $\boldsymbol{AUC}_p$ is an effective criterion for selecting high-performing models. For example, the correlation for the RSNA dataset is particularly high, with values of 0.9788 for CutPaste, 0.9846 for FPI, and 0.9960 for NSA, suggesting that $\boldsymbol{AUC}_p$ consistently identifies models that achieve superior $\boldsymbol{AUC}$ performance. Similarly, strong correlations were observed in the VIN (0.9701 for CutPaste, 0.9838 for FPI) and ISIC (0.9243 for CutPaste, 0.9274 for FPI) datasets. However, for datasets such as C16 and Brats21, the correlation was somewhat lower, particularly for FPI Poisson (0.8695 for C16 and 0.8101 for Brats21). Overall, the findings reinforce that $\boldsymbol{AUC}_p$ is a reliable model-selection method, showing a strong alignment with $\boldsymbol{AUC}$ across a range of self-supervised methods and medical datasets.

\begin{table}
\centering
\caption{This table shows the correlation between $\boldsymbol{AUC}_p$ and $\boldsymbol{AUC}$ of the trained models for all the datasets. The best-performing methods are selected to calculate the correlation. It demonstrates that $\boldsymbol{AUC}_p$ has a high correlation with $\boldsymbol{AUC}$ across different methods.}
\label{tab:pearson_correlation}
\begin{tabular}{lrrrr}
\toprule
Dataset & Cut Paste & FPI & FPI Poisson & NSA \\
\midrule
   RSNA & 0.9788 & 0.9846 & 0.9465 & 0.9960 \\
    VIN & 0.9701 & 0.9838 & 0.9847 & 0.9799 \\
   ISIC & 0.9243 & 0.9274 & 0.7643 & 0.9500 \\
    LAG & 0.9480 & 0.9731 & 0.9451 & 0.8642 \\
    C16 & 0.6627 & 0.7931 & 0.8695 & 0.9190 \\
  Brain & 0.9840 & 0.9957 & 0.9887 & 0.9898 \\
  Brats21 & 0.6689 & 0.4067 & 0.8101 & 0.7882 \\
\bottomrule
\end{tabular}
\end{table}
\section{Discussion \& Limitations}

The proposed $\boldsymbol{AUC}_p$ metric does not guarantee the selection of the best possible model for inference on any dataset as it uses {\em imperfect} labels. While, it is mathematically showed that the proposed $\boldsymbol{AUC}_p$ can rank the inference models similarly to AUC when actual labels are available, this holds only when the training dataset $\boldsymbol{D}_{train}$ is large and representative enough of normal instances. The experimental results verify that it generally selects a better inference model than the popular FID metric.

We conducted empirical simulations to assess the performance of $\boldsymbol{AUC}_p$ varying both the size of the normal training set $\boldsymbol{D}_{train}$ and the prevalence of abnormal samples in the test set $\boldsymbol{D}_{test}$. The results show that when $\boldsymbol{D}_{train}$ is small or biased, $\boldsymbol{AUC}_p$ performance degrades due to inaccurate characterization of the normal distribution. Conversely, enlarging $\boldsymbol{D}_{train}$ improves stability, and as the training set grows, the $\boldsymbol{AUC}_p$ values asymptotically approach those obtained with the true ground truth. Furthermore, when the prevalence of abnormal samples in $\boldsymbol{D}_{test}$ approaches zero, $\boldsymbol{AUC}_p$ deteriorates sharply, converging to 0.5, which is equivalent to random guessing. This sensitivity is consistent with Proposition~\ref{prop_aucp:gt_vs_gtp}, where the reliability of $\boldsymbol{AUC}_p$ depends on $n \gg k$. Our empirical analysis confirms that both factors: the representativeness of $\boldsymbol{D}_{train}$ and the presence of sufficient abnormal samples in $\boldsymbol{D}_{\mathrm{test}}$ are critical for robust model selection. The detailed output of the empirical simulation is demonstrated in Fig. \ref{fig:aucp_simulation}.

\begin{figure*}[ht]
    \centering
    \includegraphics[width=0.7 \linewidth]{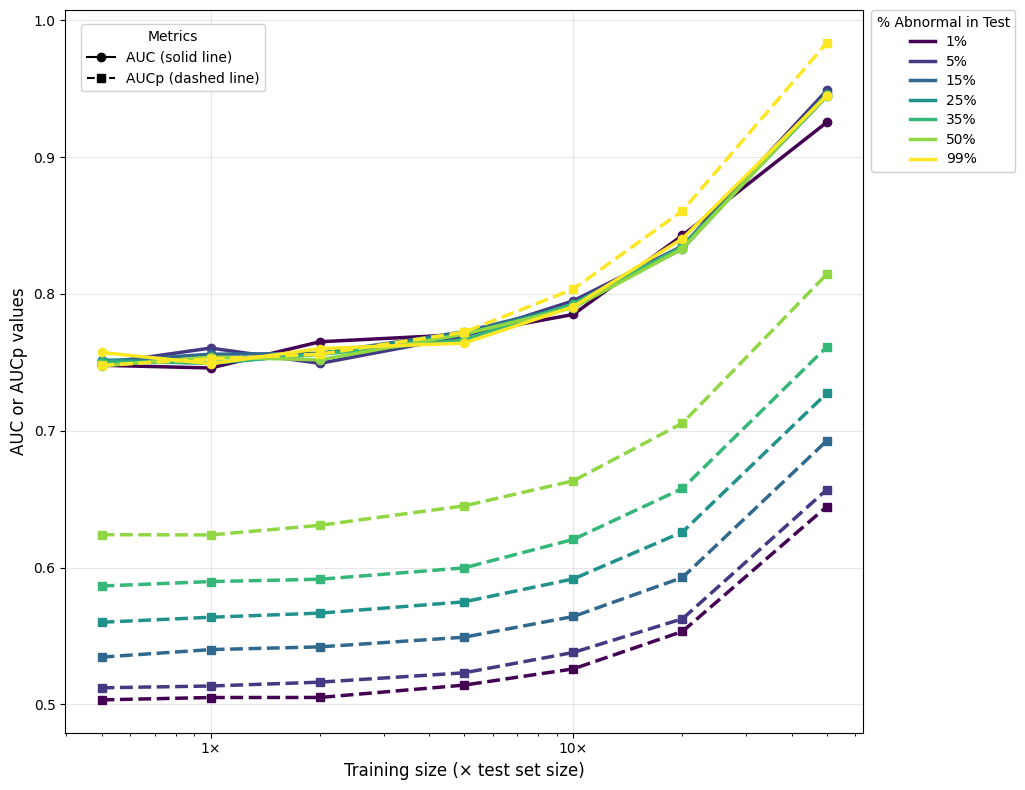}
    \caption{This figure reports empirical simulation results of $\boldsymbol{AUC}_p$ under different conditions. We vary the percentage of abnormal samples in the test set and the size of the normal-only training set and measure both $\boldsymbol{AUC}$ and $\boldsymbol{AUC}_p$. As predicted by the theory, when the test-set anomaly rate is very small, $\boldsymbol{AUC}_p$ approaches 0.5 (i.e., random guess) due to the pseudo-labeling assumption that marks all test samples as abnormal. Increasing the number of normal samples in the training set raises $\boldsymbol{AUC}_p$ in tandem with $\boldsymbol{AUC}$, reflecting improved separability under the pseudo-label regime.}
    \label{fig:aucp_simulation}
\end{figure*}

Though  $\boldsymbol{AUC}_p$ is proposed for abnormality detection from medical images, it can be utilized in abnormality detection tasks beyond medical images as long as (1) there is availability of a large and representative set of normal instances and (2) an unlabeled test set with at least one abnormal instance is present in it. If these conditions hold, $\boldsymbol{AUC}_p$ will satisfy Proposition 1 and Lemma 1 for any abnormality detection methods, irrespective of the domain. It is important to acknowledge that several other unsupervised model selection approaches have been explored in the literature, including methods based on model uncertainty, stability of learned representations, and pseudo-labeling strategies. Our formulation of $\boldsymbol{AUC}_{p}$ is complementary to these approaches: since it operates directly on inference outputs, it can, in principle, be combined with uncertainty estimation or stability analysis to further enhance model selection in scenarios without labeled validation data. This positions $\boldsymbol{AUC}{p}$ not as a replacement, but as a flexible and general criterion that can be integrated with existing model selection paradigms across domains.

Although $\boldsymbol{AUC}_p$ provides a principled and effective strategy for inference-time model selection in the absence of labeled validation data, several practical considerations should be noted. Computing $\boldsymbol{AUC}_p$ requires evaluating anomaly scores for all samples in the unlabeled evaluation set at each candidate checkpoint, analogous to standard $\boldsymbol{AUC}$-based evaluation protocols. Consequently, the computational cost scales approximately linearly with the number of checkpoints considered and the size of the evaluation set. Importantly, this cost is not unique to $\boldsymbol{AUC}_p$ but is shared by any approach that relies on $\boldsymbol{AUC}$ computation for model selection or validation. In practice, this overhead can be mitigated by checkpoint subsampling or early stopping strategies, and our empirical results demonstrate that $\boldsymbol{AUC}_p$ remains feasible and effective across diverse abnormality detection paradigms and dataset scales commonly encountered in medical imaging.
\section{Conclusion}

In the realm of abnormality detection, the absence of labeled abnormal data poses a significant challenge, extending beyond mere model training. The complexity is further compounded when selecting a well-trained and generalized model for inference, especially in the absence of labeled normal and abnormal samples in the validation dataset. While some studies rely on annotated validation datasets, others resort to synthetic abnormalities, both of which fall short of real-world applicability. Complicating matters, selecting the inference model solely based on the best training loss yields inconsistent results across different evaluation metrics. Our exploration into this challenge led us to propose a novel metric, $\boldsymbol{AUC}_p$. By assuming \textit{imperfect} labels for unlabeled instances in the test set as ``abnormal'' and leveraging known labels for normal instances in the training set, $\boldsymbol{AUC}_p$ offers a practical means of model selection in the absence of labeled validation data. Through mathematical analysis, we substantiate the efficacy of $\boldsymbol{AUC}_p$, demonstrating its proximity to actual AUC scores with representative training data. Through extensive empirical validation, we illustrate the superior correlation of $\boldsymbol{AUC}_p$ with the performance of underlying disease detection compared to traditional metrics such as FID and approaches such as selecting a model with the best training loss. Moreover, it significantly improves the detection capabilities of abnormality detection methods of both unsupervised and self-supervised methods, extending the boundaries of existing state-of-the-art methods. The proposed metric can improve both unsupervised and self-supervised disease detection performance and provide a way for more reliable and effective abnormality detection in medical imaging and beyond.

\bibliographystyle{unsrt}
\bibliography{reference}

\end{document}